%% file: main.tex
\documentclass[twoside]{article}

\usepackage[accepted]{aistats2021}

\usepackage[utf8]{inputenc} 
\usepackage[T1]{fontenc}    
\usepackage[usenames,dvipsnames]{xcolor}
\usepackage[colorlinks=true, linkcolor=black, citecolor=blue!70!black,urlcolor=black,breaklinks=true]{hyperref}
\usepackage{url}            
\usepackage{booktabs}       
\usepackage{amsfonts}       
\usepackage{nicefrac}       
\usepackage{microtype}      
\usepackage{color}

\usepackage{smile}
\newcommand{\tran}{{t}}
\def\##1\#{\begin{align}#1\end{align}}
\def\$#1\${\begin{align*}#1\end{align*}}
\usepackage{threeparttable}
\newcommand{\squishlist}{
\begin{list}{{{\small{$\bullet$}}}}
{\setlength{\itemsep}{3pt}      \setlength{\parsep}{1pt}
\setlength{\topsep}{1pt}       \setlength{\partopsep}{0pt}
\setlength{\leftmargin}{1em} \setlength{\labelwidth}{1em}
\setlength{\labelsep}{0.5em} } }
\newcommand{\squishend}{  \end{list}  }

\usepackage{graphicx}

\usepackage{natbib}
\usepackage{authblk}

\bibpunct{(}{)}{;}{a}{,}{,}

\begin{document}

\runningauthor{Wei, Fu, Liu, Li, Yang, and  Wang}

\twocolumn[

\aistatstitle{Sample Elicitation}

\aistatsauthor{Jiaheng Wei$^*$ \And Zuyue Fu$^*$  \And  Yang Liu}
\aistatsaddress{UC Santa Cruz\\jiahengwei@ucsc.edu \And Northwestern University \\zuyue.fu@u.northwestern.edu  \And UC Santa Cruz\\yangliu@ucsc.edu}%

\aistatsauthor{Xingyu Li \And Zhuoran Yang \And Zhaoran Wang }
\aistatsaddress{UC Santa Cruz\\xli279@ucsc.edu \And Princeton University\\zy6@princeton.edu \And Northwestern University \\zhaoranwang@gmail.com}
]

\input{introduction}

\input{paper}
\bibliography{graphbib,myref,library,exp}

\newpage
\onecolumn
\appendix
\input{supplement}

\end{document}

%% file: introduction.tex

\begin{abstract}
It is important to collect credible training samples $(x,y)$ for building data-intensive learning systems (e.g., a deep learning system). 
Asking people to report complex distribution $p(x)$, though theoretically viable, is challenging in practice. This is primarily due to the cognitive loads required for human agents to form the report of this highly complicated information. 
While classical elicitation mechanisms apply to eliciting a complex and generative (and continuous) distribution $p(x)$, we are interested in eliciting samples $x_i \sim p(x)$ from agents directly. We coin the above problem \emph{sample elicitation}. This paper introduces a deep learning aided method to incentivize credible sample contributions from self-interested and rational agents. 
We show that with an accurate estimation of a certain $f$-divergence function we can achieve approximate incentive compatibility in eliciting truthful samples. We then present an efficient estimator with theoretical guarantees via studying the variational forms of the $f$-divergence function. We also show a connection between this sample elicitation problem and $f$-GAN, and how this connection can help reconstruct an estimator of the distribution based on collected samples. Experiments on synthetic data, MNIST, and CIFAR-10 datasets demonstrate that our mechanism elicits truthful samples. Our implementation is available at \url{ https://github.com/weijiaheng/Credible-sample-elicitation.git}.

\end{abstract}

\section{Introduction}

The availability of a large number of credible samples is crucial for building high-fidelity machine learning models. This is particularly true for deep learning systems that are data-hungry. Arguably, the most scalable way to collect a large amount of training samples is to crowdsource from a decentralized population of agents who hold relevant data. The most popular example is the build of ImageNet \citep{deng2009imagenet}. 

The main challenge in eliciting private information is to properly score reported information such that the self-interested agent who holds private information will be incentivized to report truthfully.
Most of the existing works focused on eliciting simple categorical information, such as binary labels or multi-class categorical information. These solutions are not properly defined or scalable to more continuous or high-dimensional information elicitation tasks. For example, suppose that we are interested in collecting the calorie information of a set of food pictures, we can crowdsource to ask crowd workers to tell us how many calories are there in each particular image of food (how many calories in the hot dog shown in the picture). There exists no computable elicitation/scoring mechanism for reporting this more continuous spectrum of data (reported calorie).

\vskip5pt
In this work \footnote{Correspondence to: \{yangliu, jiahengwei\}@ucsc.edu.}, we aim to collect credible samples from self-interested agents via studying the problem of \emph{sample elicitation}. Instead of asking each agent to report the entire distribution $p$, we hope to elicit samples drawn from the distribution $\PP$ truthfully. We consider the samples $x_p \sim \PP$ and $x_q \sim \QQ$. In analogy to strictly proper scoring rules\footnote{Our specific formulation and goal will be different in details.}, we aim to design a score function $S$ s.t. $\E_{x \sim \PP}[S(x_p,x')] > \E_{x \sim \PP}[S(x_q,x')]$ for any $q \neq p$, where $x'$ is a reference answer that can be defined using elicited reports.

Our challenge lies in accurately evaluating reported samples. We first observe that the $f$-divergence function between two properly defined distributions of the samples can serve the purpose of incentivizing truthful reports of samples.
We proceed with using deep learning techniques to solve the score function design problem via a data-driven approach.
We then propose a variational approach that enables us to estimate the divergence function efficiently using reported samples, via a variational form of the $f$-divergence function, through a deep neural network. These estimation results help us establish approximate incentive compatibility in eliciting truthful samples. It is worth noting that our framework also generalizes to the setting where there is no access to ground truth samples and we can only rely on reported samples. There we show that our estimation results admit an approximate Bayesian Nash Equilibrium for agents to report truthfully. 
Furthermore, in our estimation framework, 
we use a generative adversarial approach to reconstruct the distribution from the elicited samples. In addition to the analytical results, we demonstrate the effectiveness of our mechanism in eliciting truthful samples empirically using MNIST and CIFAR-10 datasets.

We want to emphasize that the deep learning based estimators considered above can handle complex data. With our deep learning solution, we are further able to provide estimates for the divergence functions used for our scoring mechanisms with provable finite sample complexity. In this paper, we focus on developing theoretical guarantees - other parametric families either can not handle complex data, e.g., it is hard to handle images using kernel methods, or do not have provable guarantees on the sample complexity.

Our results complement the elicitation task in crowdsourcing by providing a method to elicit feature data $X \sim \PP(X|Y)$, as compared to previous works mainly focusing on eliciting labels $Y \sim \PP(Y|X)$. The difference is previous works focus on eliciting a label for a particular image, but our method enables elicitation of an image in response to a particular label (suppose that we are interested in collecting training images that contain “Cats”), which is inherently a more complex piece of information to evaluate and score with. This capability can help us build high-quality datasets for more comprehensive applications from scratch.

\noindent{\bf Related work.}
The most relevant literature to our paper is \emph{strictly proper scoring rules} and \emph{property elicitation}. Scoring rules were developed for eliciting truthful prediction (probability) \citep{Brier:50,Win:69,Savage:71,Matheson:76,Jose:06,Gneiting:07}. 
Characterization results for strictly proper scoring rules are given in \cite{McCarthy:56,Savage:71,Gneiting:07}. Property elicitation notices the challenge of eliciting complex distributions \citep{lambert2008eliciting,steinwart2014elicitation,frongillo2015vector}. For instance, \cite{abernethy2012characterization} characterize the score functions for eliciting linear properties, and \cite{frongillo2015elicitation} study the complexity of eliciting properties. Another line of relevant research is peer prediction, where solutions can help elicit private information when the ground truth verification might be missing \citep{de2016incentives,gao2016incentivizing,kong2016putting,kong2018water,kong2019information}. 
Our work complements the information elicitation literature via studying the question of sample elicitation using a variational approach to estimate $f$-divergences. A parallel work has also studied the variational approach for eliciting truthful information \cite{schoenebeck2020learning}. Our work focuses more on formalizing the sample elicitation problem. In addition, we provide sample complexity guarantees to our theorems by offering deep neural network-aided estimators and contribute to the community a practical solution.

Our work is also related to works on divergence estimation.  The simplest way to estimate divergence starts with the estimation of the density function \citep{wang2005divergence, lee2006estimation, wang2009divergence, zhang2014nonparametric, han2016minimax}.  Another method based on the variational form \citep{donsker1975asymptotic} of the divergence function comes into play \citep{broniatowski2004parametric, broniatowski2009parametric, nguyen2010estimating, kanamori2011f, ruderman2012tighter, sugiyama2012density}, where the estimation of divergence is modeled as the estimation of density ratio between two distributions.    
  The variational form of the divergence function also motivates the well-known Generative Adversarial Network (GAN) \citep{goodfellow2014generative}, which learns the distribution by minimizing the Kullback-Leibler divergence.  Follow-up works include \cite{nowozin2016f, arjovsky2017wasserstein, gulrajani2017improved, bellemare2017cramer}, with theoretical analysis in \cite{liu2017approximation, arora2017generalization, liang2018well, gao2019generative}.  See also \cite{gao2017density, bu2018estimation} for this line of work.

\noindent{\bf Notations.} 
For the distribution $\PP$, we denote by $\PP_n$ the empirical distribution given a set of samples $\{x_i\}_{i = 1}^n$ following $\PP$, i.e., $\PP_n = 1/n\cdot \sum_{i = 1}^n \delta_{x_i}$, where $\delta_{x_i}$ is the Dirac measure at $x_i$. 
We denote by $\|v\|_s = (\sum_{i = 1}^d |v^{(i)}|^s)^{1/s}$ the $\ell_s$ norm of the vector $v\in\RR^d$ where $1\leq s< \infty$ and $v^{(i)}$ is the $i$-th entry of $v$. We also denote by $\|v\|_\infty = \max_{1\leq i\leq d} |v^{(i)}|$ the $\ell_\infty$ norm of $v$. 
For any real-valued continuous function $f\colon \cX\to \RR$, we denote by $\|f\|_{L_s(\PP)}:=[\int_\cX |f(x)|^s \ud \PP]^{1/s}$ the $L_s(\PP)$ norm of $f$ and $\|f\|_s:=[\int_\cX |f(x)|^s \ud \mu]^{1/s}$ the $L_s(\mu)$ norm of $f(\cdot)$, where $\mu$ is the Lebesgue measure.  Also, we denote by $\|f\|_\infty = \sup_{x\in\cX} |f(x)|$ the $L_\infty$ norm of $f(\cdot)$. 
For any real-valued functions $g(\cdot)$ and $h(\cdot)$ defined on some unbounded subset of the real positive numbers, such that $h(\alpha)$ is strictly positive for all large enough values of $\alpha$, we write $g(\alpha)\lesssim h(\alpha)$ and $g(\alpha) = \cO(h(\alpha))$ if $|g(\alpha)| \leq c\cdot h(\alpha)$ for some positive absolute constant $c$ and any $\alpha > \alpha_0$, where $\alpha_0$ is a real number.  We denote by $[n]$ the set $\{1, 2, \ldots, n\}$.

%% file: paper.tex

\section{Preliminary}

\subsection{Sample Elicitation}
We consider two scenarios. We start with an easier case where we, as the mechanism designer, have access to a certain number of group truth samples. 
Then we move to the harder case where the inputs to our mechanism can only be elicited samples from agents. 

\paragraph{Multi-sample elicitation with ground truth samples.}
Suppose that the agent holds $n$ samples, with each of them independently drawn from $\PP$, i.e., $x_i \sim \PP$ \footnote{Though we use $x$ to denote the samples we are interested in, $x$ potentially includes both the feature and labels $(x,y)$ as in the context of supervised learning.} for $i\in[n]$. The agent can report each sample arbitrarily, which is denoted as $r_i(x_i): \Omega \rightarrow \Omega$. There are $n$ data $\{x^*_i\}_{i\in[n]}$ independently drawn from the ground truth distribution $\QQ$\footnote{The number of ground truth samples can be different from $n$, but we keep them the same for simplicity of presentation. It will mainly affect the terms $\delta$ and $\epsilon$ in our estimations.}. $x_i$s and $x_i^*$s are often correlating with each other. For example, $x^*_i$ corresponds to the true calorie level of the food contained in a picture. $x_i$ is the corresponding guess from the agent, possibly as a (randomized) function of $x^*_i$. Therefore the two distributions $\PP$ and $\QQ$ are not independent in general. 

We are interested in designing a score function $S(\cdot)$ that takes inputs of each $r_i(\cdot)$ and $\{r_j(x_j),x^*_j\}_{j\in[n]}$: $S(r_i(x_i), \{r_j(x_j),x^*_j\}_{j\in[n]})$ such that if the agent believes that $x^*$ is drawn from the same distribution $x^* \sim \PP$, then for any $\{r_j(\cdot)\}_{j\in[n]}$, it holds with probability at least $1-\delta$:
\$
& \sum_{i=1}^n \E_{x, x^* \sim \PP}\Bigl [S\bigl (x_i, \{x_j,x^*_j\}_{j\in[n]}\bigr) \Bigr]   \geq \\
 &\quad \sum_{i=1}^n \E_{x, x^* \sim \PP}\Bigl [S\bigl(r_i(x_i), \{r_j(x_j),x^*_j\}_{j\in[n]}\bigr)\Bigr]- n\cdot \epsilon.
\$
We name the above as \textbf{$(\delta, \epsilon)$-properness} (per sample) for sample elicitation. When $\delta = \epsilon=0$, it is reduced to the one that is similar to the properness definition in scoring rule literature \citep{Gneiting:07}. We also shorthand $r_i = r_i(x_i)$ when there is no confusion. Agent believes that her samples are generated from the same distribution as that of the ground truth samples, i.e., $\PP$ and $\QQ$ are the same distributions.

\paragraph{Sample elicitation with peer samples.}
Suppose there are $n$ agents each holding a sample $x_i \sim \PP_i$, where the distributions $\{\PP_i\}_{i\in[n]}$ are not necessarily the same - this models the fact that agents can have subjective biases or local observation biases. This is a more standard peer prediction setting. We denote by their joint distribution as $\PP = \PP_1 \times \PP_2 \times .... \times \PP_n$.

Similar to the previous setting, each agent can report her sample arbitrarily, which is denoted as $r_i(x_i): \Omega \rightarrow \Omega$ for any $i\in[n]$. We are interested in designing and characterizing a score function $S(\cdot)$ that takes inputs of each $r_i(\cdot)$ and $\{r_j(x_j)\}_{j \neq i}$: $S (r_i(x_i), \{r_j(x_j)\}_{j \neq i})$ such that for any $\{r_j(\cdot)\}_{j\in[n]}$, it holds with probability at least $1-\delta$ that
\$
& \E_{x \sim \PP}\Bigl [S\bigl(x_i, \{r_j(x_j) = x_j\}_{j \neq i}\bigr)\Bigr]   \geq \\
 &\quad \E_{x \sim \PP}\Bigl[S\bigl(r(x_i), \{r_j(x_j) = x_j\}_{j \neq i}\bigr)\Bigl ] - \epsilon.
\$
We name the above as \textbf{$(\delta, \epsilon)$-Bayesian Nash Equilibrium} (BNE) in truthful elicitation. We only require that agents are all aware of the above information structure as common knowledge, but they do not need to form beliefs about details of other agents' sample distributions. Each agent's sample is private to herself. 

\paragraph{Connection to the proper scoring rule} At a first look, this problem of eliciting quality data is readily solvable with the seminal solution for eliciting distributional information, called the strictly proper scoring rule \citep{Brier:50,Win:69,Savage:71,Matheson:76,Jose:06,Gneiting:07}: suppose we are interested in eliciting information about a random vector $X = (X_1,...,X_{d-1},Y) \in \Omega \subseteq \mathbb R^{d}$, whose probability density function is denoted by $p$ with distribution $\PP$. As the mechanism designer, if we have a sample $x$ drawn from the true distribution $\PP$, we can apply strictly proper scoring rules to elicit $p$: the agent who holds $p$ will be scored using $S(p, x)$. $S$ is called strictly proper if it holds for any $p$ and $q$ that $\E_{x \sim \PP}[S(p,x)] > \E_{x \sim \PP}[S(q,x)]$.  The above elicitation approach has two main caveats that limited its application: (1) When the outcome space $|\Omega|$ is large and is even possibly infinite, it is practically impossible for any human agents to report such a distribution with reasonable efforts. Consider the example where we are interested in building an image classifier via first collecting a certain category of  high-dimensional image data. While classical elicitation results apply to eliciting a complex, generative and continuous distribution $p(x)$ for this image data, we are interested in eliciting samples $x_i \sim p(x)$ from agents. 
(2) The mechanism designer may not possess any ground truth samples.

\subsection{$f$-divergence} \label{sec:f_div_intro}

It is well known that maximizing the expected proper scores is equivalent to minimizing a corresponding Bregman divergence \citep{Gneiting:07}. More generically, we take the perspective that divergence functions have great potentials to serve as score functions for eliciting samples. We define the $f$-divergence between two distributions $\PP$ and $\QQ$ with probability density function $p$ and $q$ as
\#\label{eq:def-div}
D_{f}(q\|p) = \int p(x) f\biggl(\frac{q(x)}{p(x)}\biggr)\ud\mu.
\#
Here $f(\cdot )$ is a function satisfying certain regularity conditions, which will be specified later. 
Solving our elicitation problem involves evaluating the $D_f(q\|p)$ successively based on the distributions $\PP$ and $\QQ$, without knowing the probability density functions $p$ and $q$. Therefore, we have to resolve to a form of $D_f(q\|p)$ which does not involve the analytic forms of $p$ and $q$, but instead sample forms. Following from Fenchel's convex duality, it holds that
\#\label{eq:hahadf}
D_{f}(q\|p) = \max_{\tran(\cdot)} \EE_{x\sim \QQ}[ t(x) ] - \EE_{x\sim\PP}[f^\dagger(t(x))],
\#
where $f^\dagger(\cdot)$ is the Fenchel duality of the function $f(\cdot)$, which is defined as $f^\dagger(u) = \sup_{v\in\RR} \{uv -f(v) \}$, and the max is taken over all functions $t(\cdot)\colon \Omega\subset \RR^d \to \RR$.

\section{Sample Elicitation: A Variational Approach}\label{eq:sec-se}

Recall from \eqref{eq:hahadf} that $D_f(q\|p)$ admits the following variational form:
\#\label{eq:v-form-exact}
D_f(q\|p) = \max_{\tran(\cdot)} \E_{x \sim \QQ} [\tran(x)] - \E_{x \sim \PP}[ f^{\dag} (\tran(x))]. 
\#
We highlight that via functional derivation, \eqref{eq:v-form-exact} is solved by $t^*(x;p,q) = f'(\theta^*(x;p,q))$, where $\theta^*(x;p,q) = q(x)/p(x)$ is the density ratio between $p$ and $q$.  Our elicitation builds upon such a variational form \eqref{eq:v-form-exact} and the following estimators,
\begin{align*}
& \hat{\tran}(\cdot;p,q) = \argmin_{\tran(\cdot)}\E_{x \sim \PP_n}[ f^{\dag} (\tran(x))] -  \E_{x \sim \QQ_n} [\tran(x)], \\
&\hat{D}_f(q\|p) =  \E_{x \sim \QQ_n} [\hat{\tran}(x)] - \E_{x \sim \PP_n}[ f^{\dag} (\hat{\tran}(x))]. 
\end{align*}

\subsection{Error Bound and Assumptions}
Suppose we have the following error bound for estimating $D_f(q\|p)$: for any probability density functions $p$ and $q$, it holds with probability at least $1-\delta(n)$ that
\begin{align}\label{eq:eps2}
\left|\hat{D}_f(q\|p)  - D_f(q\|p)\right|& \leq \epsilon(n),
\end{align}
where $\delta(n)$ and $\epsilon(n)$ will be specified later in Section \S \ref{section:subproblem_fdiv}. To obtain such an error bound, we need the following assumptions. 

\begin{assumption}[Bounded Density Ratio]\label{assum:ratio}
The density ratio $\theta ^*(x;p,q) = q(x)/p(x)$ is bounded such that  $0<\theta_0\leq \theta^* \leq \theta_1$ holds for positive absolute constants $\theta_0$ and $\theta_1$.
\end{assumption}

The above assumption is standard in related literature \citep{nguyen2010estimating, suzuki2008approximating}, which requires that the probability density functions $p$ and $q$ lie on the same support.  For simplicity of presentation, we assume that this support is $\Omega\subset\RR^d$. We define the $\beta$-H\"older function class on $\Omega$ as follows. 
\begin{definition}[$\beta$-H\"older Function Class]\label{definition:beta_holder}
The  $\beta$-H\"older function class with radius $M$ is defined as
\$
 \cC_d^\beta(\Omega,M) =& \biggl\{ \tran(\cdot)\colon\Omega\subset \RR^d\to \RR\colon  \sum_{\|\alpha\|_1 < \beta} \|\partial^\alpha t\|_\infty   \\&+ \sum_{\|\alpha\|_1 = \lfloor \beta \rfloor } \sup_{\substack{x,y\in \Omega, \\x\neq y}}  \frac{|\partial^\alpha t(x) - \partial ^\alpha t(y)|}{\|x-y\|_\infty^{\beta - \lfloor \beta \rfloor}} \leq M \biggr\},
\$
where $\partial^\alpha = \partial ^{\alpha_1}\cdots \partial^{\alpha_d}$ with $\alpha = (\alpha_1,\ldots, \alpha_d)\in\NN^d$. 
\end{definition}

We impose the following assumptions. 
\begin{assumption}[$\beta$-H\"older Condition]\label{assum:beta-holder}
The function $t^*(\cdot; p,q)\in\cC_d^\beta(\Omega,M)$ for some positive absolute constants $M$ and $\beta$, where $\cC_d^\beta(\Omega,M)$ is the $\beta$-H\"older function class in Definition \ref{definition:beta_holder}.
\end{assumption}

\begin{assumption}[Regularity of Divergence Function]\label{assumption:reg}
The function $f(\cdot)$ is smooth on $[\theta_0, \theta_1]$ and $f(1) = 0$. Also, $f$ is $\mu_0$-strongly convex, and has $L_0$-Lipschitz continuous gradient on $[\theta_0, \theta_1]$, where $\mu_0$ and $L_0$ are positive absolute constants, respectively. 
\end{assumption}
We highlight that we only require that the conditions in Assumption \ref{assumption:reg} hold on the interval $[\theta_0, \theta_1]$, where the absolute constants $\theta_0$ and $\theta_1$ are specified in Assumption \ref{assum:ratio}.  Thus, Assumption \ref{assumption:reg} is mild and it holds for many commonly used functions in the definition of $f$-divergence. For example, in Kullback-Leibler (KL) divergence, we take $f(u) = -\log u$, which satisfies Assumption \ref{assumption:reg}.

We will show that under Assumptions \ref{assum:ratio}, \ref{assum:beta-holder}, and \ref{assumption:reg}, the bound \eqref{eq:eps2} holds.  See Theorem \ref{thm:gen} in  Section \S \ref{section:subproblem_fdiv} for details.

\subsection{Multi-sample elicitation with ground truth samples}

In this section, we focus on multi-sample elicitation with ground truth samples. Under this setting, as a reminder, the agent will report multiple samples. After the agent reported her samples, the mechanism designer obtains a set of ground truth samples $\{x^*_i\}_{i\in[n]} \sim \QQ$ to serve the purpose of evaluation. This falls into the standard strictly proper scoring rule setting.

 Our mechanism is presented in Algorithm \ref{f1}. 

\begin{algorithm}[!h]
\caption{ $f$-scoring mechanism for multiple-sample elicitation with ground truth}\label{f1}
\begin{algorithmic}
\STATE{1. Compute
$
\hat{\tran}(\cdot;p, q) = \argmin_{\tran(\cdot)}\E_{x \sim \PP_n}[ f^{\dag} (\tran(x))] -  \E_{x^* \sim \QQ_n} [\tran(x^*)]. 
$\vskip-5pt}
\STATE{2. For $i\in[n]$, pay reported sample $r_i$ using
$
 S\bigl(r_i, \{r_j, x^*_j\}_{j=1}^n\bigr) :=a - b\bigl( \E_{x \sim \QQ_n} [\hat{\tran}(x; p, q)]- f^{\dag} (\hat{\tran}(r_i; p, q))\bigr)
$
for some constants $a,b>0$.} 
\end{algorithmic}
\end{algorithm}

Algorithm \ref{f1} consists of two steps: Step 1 is to compute the function $\hat{\tran}(\cdot; p, q)$, which enables us, in Step 2, to pay agent using a linear-transformed estimated divergence between the reported samples and the true samples.  We have the following result.

\begin{theorem}\label{thm:multi}
The $f$-scoring mechanism in Algorithm \ref{f1} achieves $(2\delta(n), 2b\epsilon(n))$-properness.
\end{theorem}
The proof is mainly based on the error bound in estimating $f$-divergence and its non-negativity. Not surprisingly, if the agent believes her samples are generated from the same distribution as the ground truth sample, and that our estimator can well characterize the difference between the two sets of samples, she will be incentivized to report truthfully to minimize the difference.  We defer the proof to Section \S \ref{proof:thm:multi}.

\subsection{Single-task elicitation without ground truth samples}

The above mechanism in Algorithm \ref{f1}, while intuitive, has two caveats: 
1. The agent needs to report multiple samples (multi-task/sample elicitation);  2. Multiple samples from the ground truth distribution are needed. 
To deal with such caveats, we consider the single point elicitation in an elicitation without a verification setting. Suppose there are $2n$ agents each holding a sample $x_i \sim \PP_i$ \footnote{This choice of $2n$ is for the simplicity of presentation.}.
We randomly partition the agents into two groups and denote the joint distributions for each group's samples as $\PP$ and $\QQ$ with probability density functions $p$ and $q$ for each of the two groups. Correspondingly, there are a set of $n$ agents for each group, respectively, who are required to report their \emph{single} data point according to two distributions $\PP$ and $\QQ$, i.e., each of them holds $\{x^p_i\}_{i\in[n]} \sim \PP$ and $\{x^q_i\}_{i\in[n]} \sim \QQ$.  As an interesting note, this is also similar to the setup of a Generative Adversarial Network (GAN), where one distribution corresponds to a generative distribution $x\given y=1$,\footnote{``$\given$'' denotes the conditional distribution.} and another $x\given y=0$. This is a connection that we will further explore in Section \S \ref{sec:gan} to recover distributions from elicited samples.

We denote by the joint distribution of $p$ and $q$ as $p\oplus q$ (distribution as $\PP\oplus \QQ$), and the product of the marginal distribution as $p \times q$ (distribution as $\PP \times \QQ$). We consider the divergence between the two distributions:
$
D_f(p\oplus q \|p \times q) =   \max_{\tran(\cdot)} \E_{\xb \sim \PP\oplus \QQ} [\tran(\xb)]   - \E_{\xb \sim \PP \times \QQ}[ f^{\dag} (\tran(\xb))]. 
$
Motivated by the connection between mutual information and KL divergence, we define generalized $f$-mutual information in the follows, which characterizes the generic connection between a generalized $f$-mutual information and $f$-divergence. 

\begin{definition}[\cite{kong2019information}]
The generalized $f$-mutual information between $p$ and $q$ is defined as 
$
I_f(p;q) = D_f(p\oplus q \|p\times  q) . 
$
\end{definition}
Further it is shown in \cite{kong2018water,kong2019information} that the data processing inequality for mutual information holds for $I_f(p;q)$ when $f$ is strictly convex. We define the following estimators, 
\begin{align}\label{eq:mi_opt}
\hat{\tran}(\cdot; p\oplus q, p\times  q) =& \argmin_{\tran(\cdot)}\E_{\xb \sim \PP_n \times \QQ_n} [ f^{\dag} (\tran(\xb))]    \notag\\ &\qquad \quad-  \E_{\xb \sim \PP_n \oplus \QQ_n} [\tran(\xb)],  \notag\\
\hat{D}_f(p\oplus q \|p\times  q) =&  \E_{\xb \sim \PP_n \oplus \QQ_n} [\hat{\tran}(\xb; p\oplus q, p\times  q)]  \notag\\ & - \E_{\xb \sim \PP_n \times \QQ_n}[ f^{\dag} (\hat{\tran}(\xb; p\oplus q, p\times  q))], 
\end{align}
where $\PP_n$ and $\QQ_n$ are empirical distributions of the reported samples. We denote $\xb \sim \PP_n \oplus \QQ_n\given r_i$ as the conditional distribution when the first variable is fixed with realization $r_i$. Our mechanism is presented in Algorithm \ref{f2}. 

\begin{algorithm}[!h]
\caption{ $f$-scoring mechanism for sample elicitation}\label{f2}
\begin{algorithmic}[1]
\STATE Compute $\hat{\tran}(\cdot; p\oplus q, p\times  q)$ as
\[
\hat{\tran}(\cdot) = \argmin_{\tran(\cdot)}\E_{\xb \sim \PP_n \times \QQ_n}[ f^{\dag} (\tran(\xb))] -  \E_{\xb \sim \PP_n \oplus \QQ_n} [\tran(\xb)]. 
\]
\STATE Pay each reported sample $r_i$ using: for some constants $a,b>0$,
\begin{small}
\begin{align*}
    S(r_i, \{r_j\}_{j \neq i}) := &a + b\biggl( \E_{\xb \sim \PP_n \oplus \QQ_n|r_i} \left[\hat{\tran}(\xb; p\oplus q, p\times  q)\right]   \\&- \E_{\xb \sim \PP_n \times \QQ_n|r_i}\left[ f^{\dag} (\hat{\tran}(\xb; p\oplus q, p\times  q))\right]
\biggr)
\end{align*}
\end{small}
\end{algorithmic}
\end{algorithm}

Similar to  Algorithm \ref{f1}, the main step in Algorithm \ref{f2} is to estimate the $f$-divergence between $\PP_n \times \QQ_n$ and $ \PP_n \oplus \QQ_n$ using reported samples. Then we pay agents using a linear-transformed form of it. 
We have the following result. 
\begin{theorem}\label{thm:single}
The $f$-scoring mechanism in Algorithm \ref{f2} achieves $(2\delta(n), 2b\epsilon(n))$-BNE.
\end{theorem}
The theorem is proved by error bound in estimating $f$-divergence, a max argument, and the data processing inequality for $f$-mutual information. We defer the proof in Section \S \ref{proof:thm:single}. 

The job left for us is to establish the error bound in estimating the $f$-divergence to obtain $\epsilon(n)$ and $\delta(n)$.  Roughly speaking, if we solve the optimization problem \eqref{eq:mi_opt} via deep neural networks with proper structure,  it holds that 
$
 \delta(n) = 1- \exp\{-n^{d/(2\beta+d)}\log^5 n\}$ and $   \epsilon(n) = c \cdot n^{-{\beta}/{(2\beta + d)}} \log^{7/2} n,
$
where $c$ is a positive absolute constant.  We state and prove this result formally in Section \S \ref{section:subproblem_fdiv}.

\begin{remark}
(1) When the number of samples grows, it holds that $\delta(n)$ and $\epsilon(n)$ decrease to 0 at least polynomially fast, and our guaranteed approximate incentive-compatibility approaches a strict one. (2) Our method or framework handles arbitrary complex information, where the data can be sampled from high dimensional continuous space. (3) The score function requires no prior knowledge. Instead, we design estimation methods purely based on reported sample data. (4)  Our framework also covers the case where the mechanism designer has no access to the ground truth, which adds contribution to the peer prediction literature. So far peer prediction results focused on eliciting simple categorical information. Besides handling complex information structures, our approach can also be viewed as a data-driven mechanism for peer prediction problems.
\end{remark}

\section{Estimation of $f$-divergence}\label{section:subproblem_fdiv}
In this section, we introduce an estimator of $f$-divergence and establish the statistical rate of convergence, which characterizes $\epsilon(n)$ and $\delta(n)$. 
 For the simplicity of presentation, in the sequel, we estimate the $f$-divergence $D_f(q\|p)$ between  distributions $\PP$ and $\QQ$ with probability density functions $p$ and $q$, respectively.  The rate of convergence of the estimated $f$-divergence can be easily extended to that of the estimated mutual information.

Following from the analysis in Section \S \ref{eq:sec-se}, by Fenchel duality,  estimating $f$-divergence between $\PP$ and $\QQ$ is equivalent to solving the following optimization problem, 
\#\label{eq:subproblem_popu}
&t^*(\cdot;p,q) = \argmin_{t(\cdot)} \EE_{x\sim \PP}[ f^\dagger(t(x))] - \EE_{x\sim \QQ}[t(x)] ,\notag\\
&D_f(q\|p) =  \EE_{x\sim\QQ}[  t^*(x;p,q) ] -\EE_{x\sim\PP}[ f^\dagger( t^*(x;p,q) )].
\#
A natural way to estimate the divergence $D_f(q\| p)$ is to solve the empirical counterpart of \eqref{eq:subproblem_popu}: 
\#\label{eq:gen-form}
&t^\natural(\cdot;p,q) = \argmin_{t\in\Phi} \EE_{x\sim\PP_n}[ f^\dagger(t(x))] - \EE_{x\sim\QQ_n}[t(x)] ,\notag\\
& D_f^\natural (q\|p) =  \EE_{x\sim\QQ_n}[ t^\natural(x;p,q) ] -\EE_{x\sim\PP_n}[ f^\dagger( t^\natural(x;p,q) )], 
\#
where $\Phi$ is a function space with functions whose infinity norm is bounded by a constant $M$.   We establish the statistical rate of convergence with general function space $\Phi$ as follows, and defer the case where $\Phi$ is a family of deep neural networks in Section \S \ref{sec:nn-app} of the appendix.  We introduce the following definition of the covering number. 

\begin{definition}[Covering Number]
Let $(V, \|\cdot \|_{L_2})$ be a normed space, and $\Phi \subset V$. We say that $\{v_1, \ldots, v_N\}$ is a $\delta$-covering over $\Phi$ of size $N$ if $\Phi\subset\cup_{i = 1}^N B(v_i, \delta)$, where $B(v_i, \delta)$ is the $\delta$-ball centered at $v_i$.  The covering number is defined as $N_2(\delta, \Phi) = \min\{ N\colon  \exists~ \delta\textrm{-covering over $\Phi$ of size $N$} \}$. 
\end{definition}

We impose the following assumption on the covering number of the space $\Phi$, which characterizes the representation power of  $\Phi$. 

\begin{assumption}\label{assum:phi-cn}
$N_2(\delta, \Phi) = \cO( \exp\{ \delta^{-\gamma_\Phi} \})$, where $0 < \gamma_\Phi < 2$. 
\end{assumption}

In the following theorem, we establish the statistical convergence rate of the estimator proposed in \eqref{eq:gen-form}. For the simplicity of discussion, we assume that $t^* \in \Phi$. 

\begin{theorem}\label{thm:gen}
Suppose that Assumptions \ref{assum:ratio}, \ref{assumption:reg}, and \ref{assum:phi-cn} hold, and $t^*(\cdot; p, q)\in \Phi$. With probability at least $1- \exp(-n^{\gamma_\Phi / (2 + \gamma_\Phi)} )$, we have
$
| D_f^\natural(q\| p) - D_f(q\| p)| \lesssim n^{-1/(\gamma_\Phi + 2)}. 
$
\end{theorem}

We defer the proof of Theorem \ref{thm:gen} to Section \S\ref{proof:thm:gen}. By Theorem \ref{thm:gen}, the estimator in \eqref{eq:gen-form} achieves the optimal non-parametric rate of convergence \citep{stone1982optimal}.

\section{Connection to $f$-GAN and Reconstruction of Distribution}\label{sec:gan}
After sample elicitation, a natural question to ask is how to learn a representative probability density function from the samples. Denote the probability density function from elicited samples as $p$. Then, learning the probability density function $p$ is to solve for
\#\label{eq:form_fgan}
q^* = \argmin_{q\in\cQ} D_{f}(q\|p),
\#
where $\cQ$ is the probability density function space.   By the non-negativity of $f$-divergence, $q^* = p$ solves  \eqref{eq:form_fgan}, which implies that by solving \eqref{eq:form_fgan}, we reconstruct the representative probability from the samples. 

To see the connection between \eqref{eq:form_fgan} and the formulation of $f$-GAN \citep{nowozin2016f}, by   \eqref{eq:hahadf} and \eqref{eq:form_fgan}, we have
$ 
q^* = \argmin_{q\in \cQ}\max_{t} \EE_{x\sim\QQ}[ t(x) ] - \EE_{x\sim\PP}[f^\dagger(t(x))],
$ 
which is the formulation of $f$-GAN. 
We now propose the following estimator, 
\#
& q^\natural = \argmin_{q\in \cQ} D_f^\natural (q\|p), \label{eq:q1-est}
\# 
where $ D_f^\natural(q\|p)$ is defined in \eqref{eq:gen-form}.   We defer the case where deep neural networks are used to construct the estimators  in Section \S \ref{sec:nn-app} of the appendix.  
We impose the following assumption.

\begin{assumption}\label{assum:cov_num1}
$N_2(\delta, \cQ) = \cO( \exp\{ \delta^{-\gamma_\Phi} \})$.
\end{assumption}

The following theorem characterizes the error bound of estimating $q^*$ by $q^\natural$.
\begin{theorem}\label{theorem:outer_prob_sparsity1} 
Under the same assumptions in Theorem \ref{thm:gen}, further if Assumption \ref{assum:cov_num1} holds, for sufficiently large sample size $n$, with probability at least $1-1/n$, we have
$$ 
D_f( q^\natural \|p)  \lesssim n^{-{1}/{(\gamma_\Phi+2)}} \cdot \log n    + \min_{\tilde q\in\cQ}D_f(\tilde q\|p).
$$ 
\end{theorem}

We defer the proof of Theorem \ref{theorem:outer_prob_sparsity1} to Section \S \ref{subsection:proof_outer_prob_sparsity1} in Appendix.  In the upper bound of Theorem \ref{theorem:outer_prob_sparsity1}, the first term characterizes the generalization error of the estimator in \eqref{eq:q1-est}, while the second term is the approximation error.

\section{Experiment results}
We use the synthetic dataset, MNIST \citep{mnist} and CIFAR-10 \citep{cifar} test dataset to validate the incentive property of our mechanism. 

\subsection{Experiments on synthetic data}
In this section, the scores are estimated based on the variational approach we documented earlier and the method used in \citep{nguyen2010estimating} for estimating the estimator. The experiments are based on synthetic data drawn from 2-dimensional Gaussian distributions \footnote{We choose simpler distribution so we can compute the scores analytically for verification purpose}. We randomly generate 2 pairs of the means ($\mu$) and covariance matrices ($\Sigma$) (see Table~\ref{Tab:exp_syn} for details.) Experiment results show that truthful reports lead to higher scores that are close to analytical MI.

\begin{table}[h]
\caption{Score comparison among truthful reports, random shift and random reports.} \label{Tab:exp_syn}
\scriptsize 
\begin{center}
\begin{tabular}{|c|c|c|}
\hline
& Exp1  &Exp2 \\
\hline 
\hline
 $\mu$  & $
\begin{pmatrix}
    -2.97\\ 8.98
\end{pmatrix}
$ &$
\begin{pmatrix}
    6.98\\ 8.39
\end{pmatrix}
$ \\  \hline 
 $\Sigma$  & $
\begin{pmatrix}
    1.28 & 4.39\\ 
    4.39 & 16.19 
\end{pmatrix}
$ & $
\begin{pmatrix}
    10.56 & 16.18\\
    16.18 & 26.43 
\end{pmatrix}
$  \\  \hline
 Analytical  & 1.30 & 1.40\\  \hline
Truthful  & \textbf{1.36}  $\pm$ 0.06 & \textbf{1.32} $\pm$ 0.05  \\  \hline
 Random shift  &1.08 $\pm$ 0.05 & 1.11  $\pm$  0.04 \\  \hline
 Random report  & 0.13 $\pm$ 0.02& 0.20  $\pm$ 0.02 \\
 \hline
\end{tabular}
\end{center}
\end{table}

Two numerical experiments are shown in Table~\ref{Tab:exp_syn}. In each of the above experiments, we draw $1000$ pairs of samples $(x_i,y_i)$ from the corresponding Gaussian distribution. This set of pairs reflect the joint distribution $\mathbb{P}\oplus\mathbb{Q}$, while the sets $\{x_i\}$ and $\{y_i\}$ together correspond to the marginal distribution $\mathbb{P}\times\mathbb{Q}$. For each experiment, we repeat ten times and calculated the mean estimated score and the corresponding standard deviation.

For simplicity, we adopted $a=0$, $b=1$, and $f(\cdot)=-\log(\cdot)$ in estimating the scores, which make the expected score nothing but the Mutual Information (MI) between distributions $\mathbb{P}$ and $\mathbb{Q}$. The analytical values and the estimated scores of the MI for the two experiments are listed in the $4$th and $5$th row in Table~\ref{Tab:exp_syn}, respectively. To demonstrate the effects of untruthful reports (misreports) on our score, we consider two types of untruthful reporting:
\squishlist
\item \textbf{Random Shift:} The agent draws random noise from the uniform distribution $U(0,3)$ and add to $\{x_i\}$.
\item \textbf{Random Report;} Agent simply reports random signals drawn from the uniform distribution $U(0,2\sigma)$, where $\sigma$ is the standard deviation of the marginal distribution $\mathbb{P}$. This models the case when agents contribute uninformative information.
\squishend

As expected, the scores of untruthful reports are generally lower than the scores of truthful ones.

\begin{figure*}[htb]
    \centering
    \subfigure[\scriptsize MNIST]{\includegraphics[width=.48\textwidth,height=4.2cm]{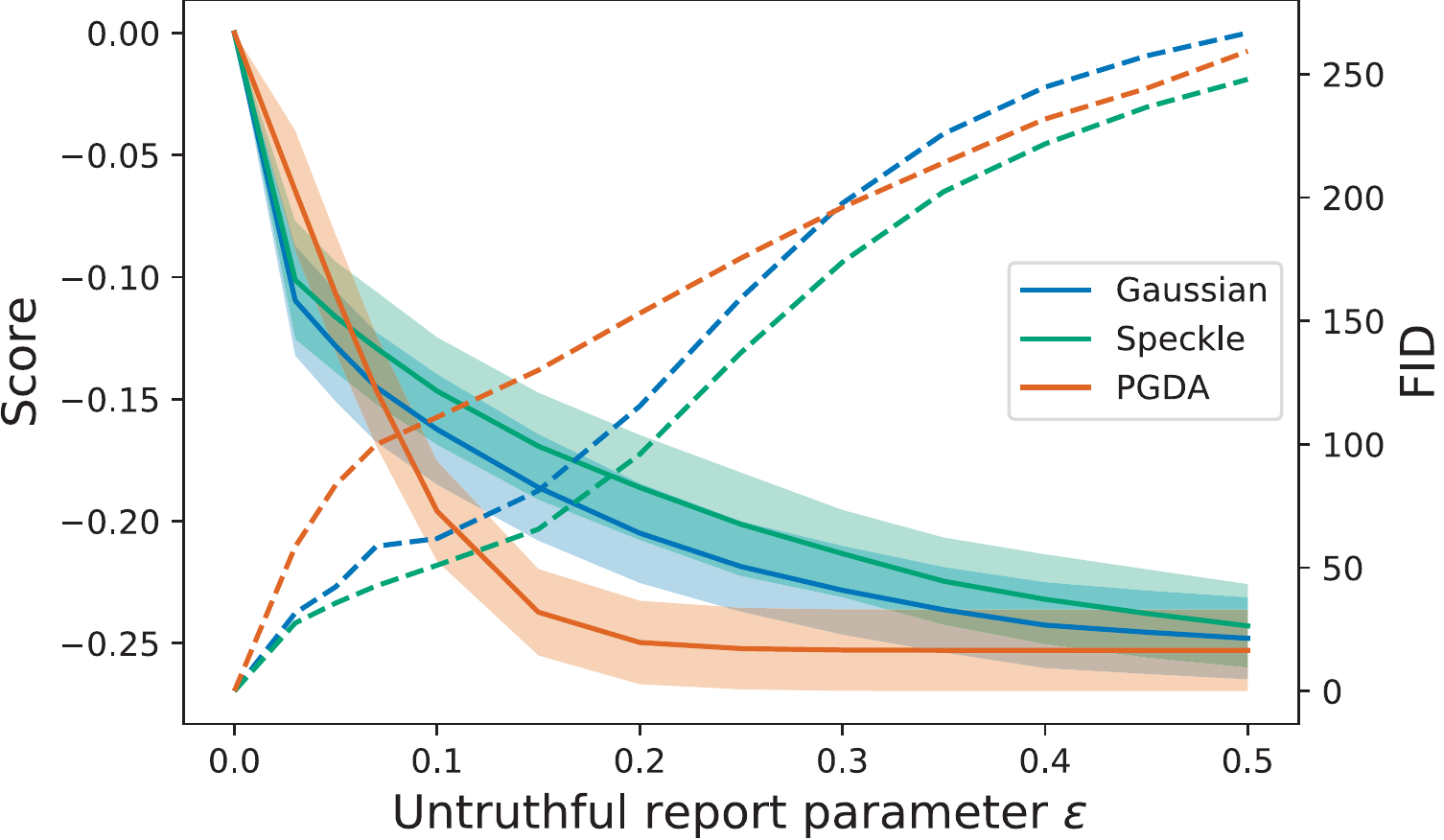}
    \label{Fig: 1a}}
    \subfigure[\scriptsize CIFAR-10]
    {\includegraphics[width=.48\textwidth,height=4.2cm]{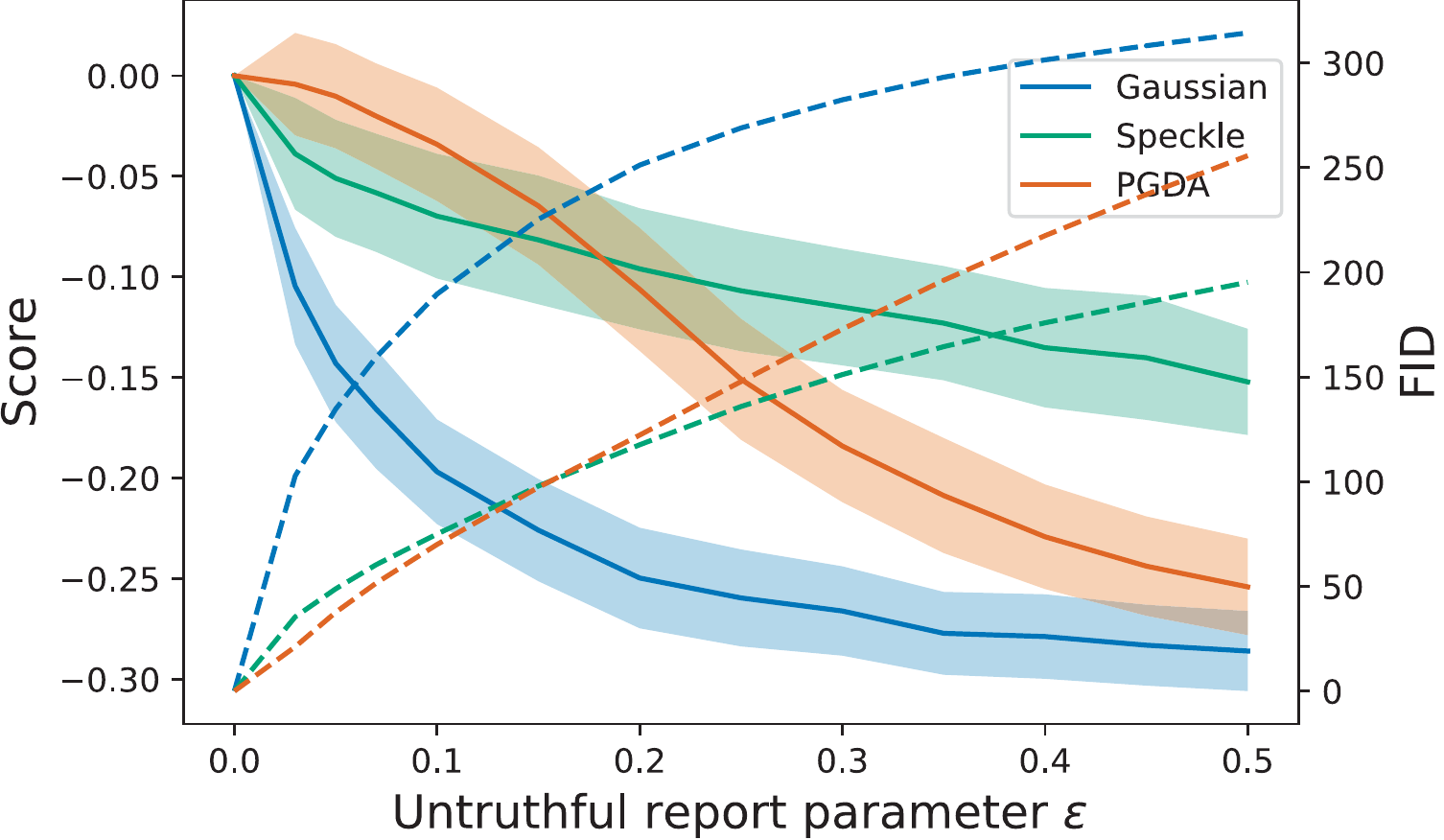}
    \label{Fig: 1b}}
        \vspace{-7pt}
        \caption{Scores and FID value w.r.t. $\epsilon$ with ground truth verification. \textbf{Dashed lines} represent FID.  
        \vspace{-5pt}
    }
    \label{Fig:fig1}
\end{figure*}

\begin{figure*}[htb]
    \centering
    \subfigure[\scriptsize MNIST]{\includegraphics[width=.47\textwidth,height=4.1cm]{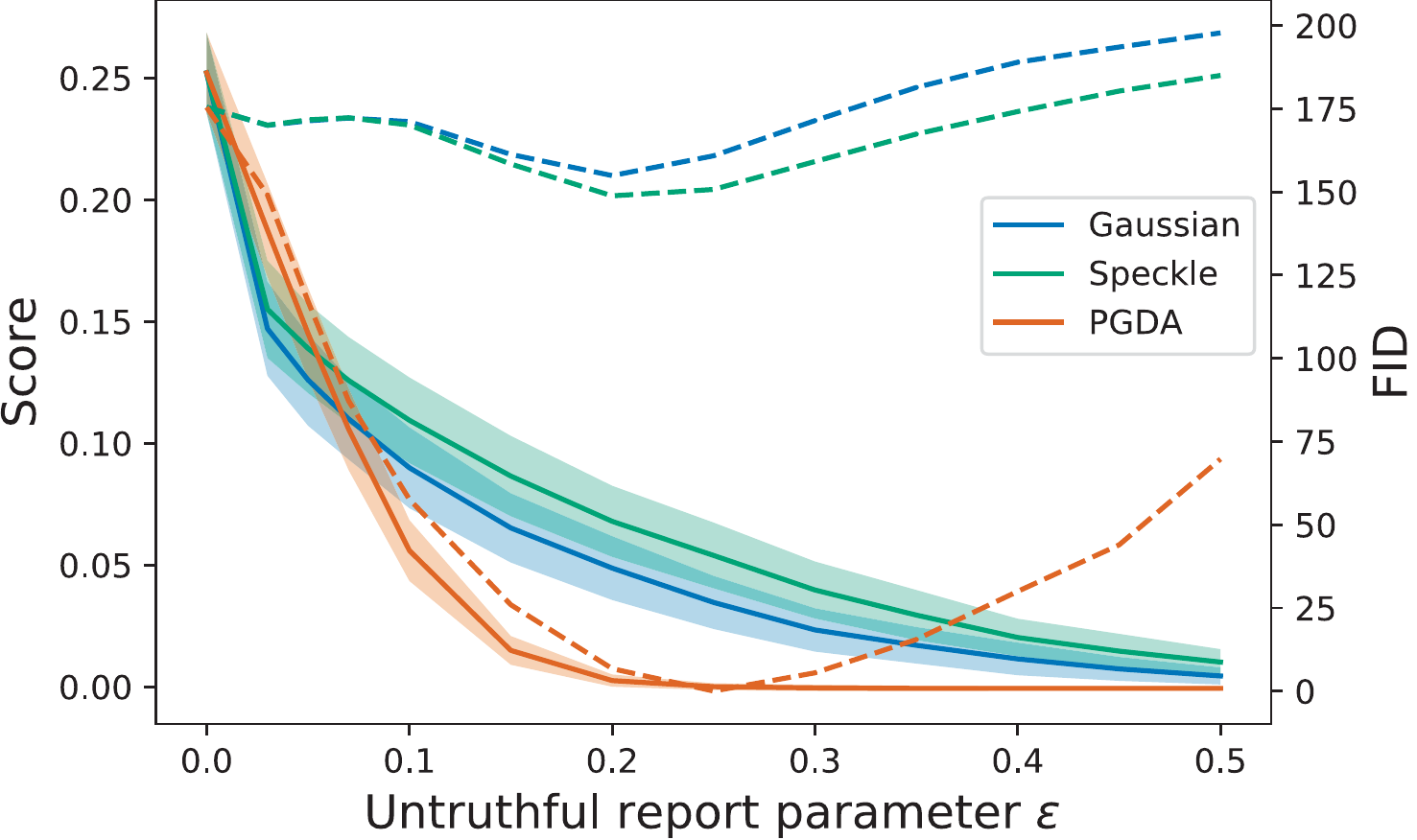}
    \label{Fig: 2a}}
    \subfigure[\scriptsize CIFAR-10]
    {\includegraphics[width=.48\textwidth,height=4.2cm]{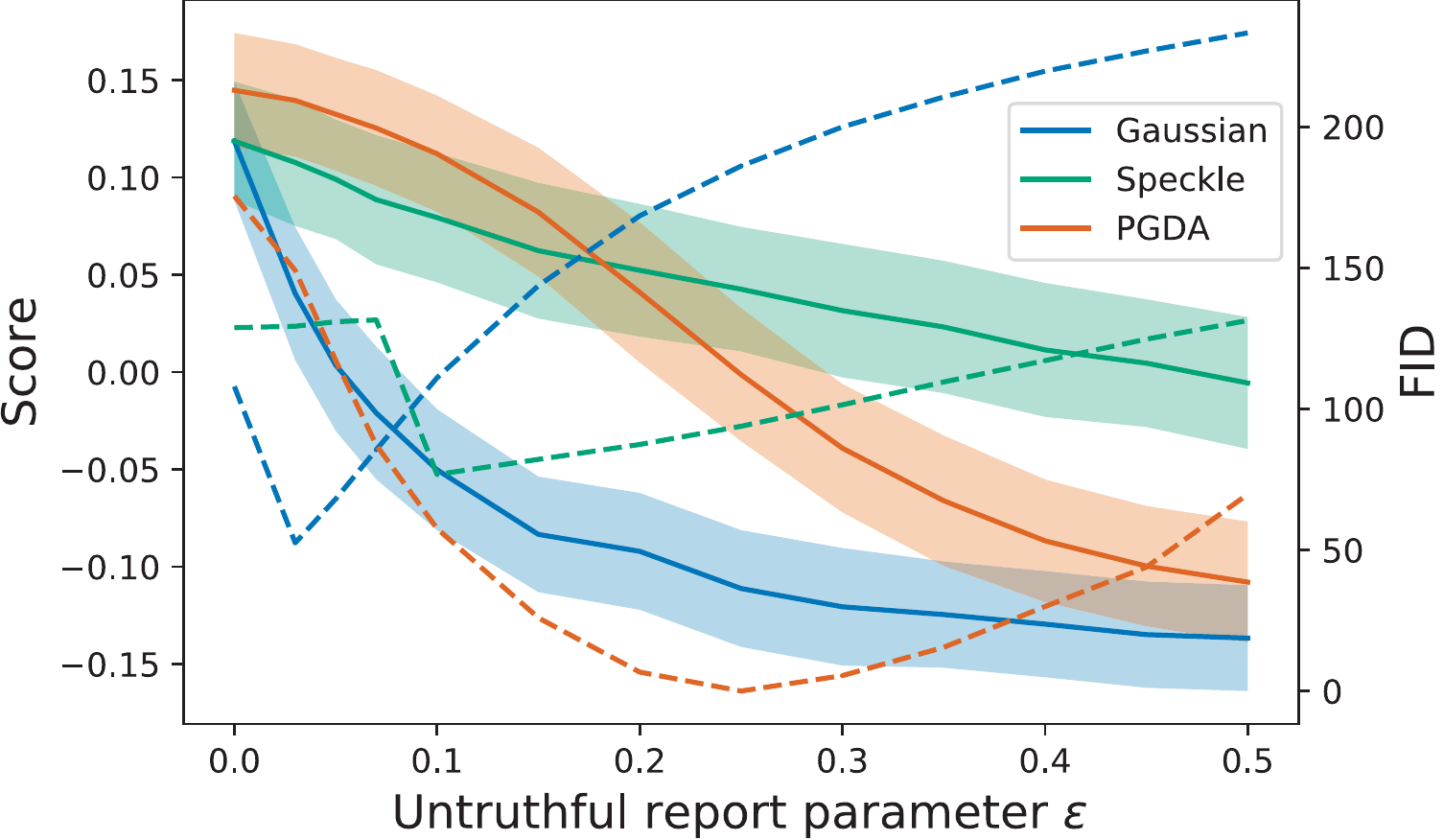}
    \label{Fig: 2b}}
        \vspace{-7pt}
        \caption{Scores and FID value w.r.t. $\epsilon$ with no ground truth verification but only peer samples. \textbf{Dashed lines} represent FID.
        \vspace{-5pt}
    }
    \label{Fig:fig2}
\end{figure*}

\subsection{Experiments on Image data}
We use the test dataset of MNIST and CIFAR-10 to further validate the robustness of our mechanism. Since the images are high-dimensional data, we choose to skip Step 1 in Algorithm \ref{f1} and \ref{f2} and instead adopt $\hat{\tran}$ and $f^{\dag}$ as suggested by \citep{nowozin2016f} (please refer to Table 6 therein). 

We take Total-Variation as an example in our experiments, and use \citep{nowozin2016f} $
\hat{\tran}(x)=f^{\dag}(\hat{\tran}(x))=\dfrac{1}{2}\tanh{(x)}.$ We adopt $a=0, b=1$ for simplicity.  For other divergences, please refer to Table \ref{Tab:f_div} in the Appendix. 

Untruthful reports are simulated by inducing the following three types of noise, with $\epsilon$ being a hyper-parameter controlling the degree of misreporting:
\squishlist
\item \textbf{Gaussian Noise:} add Gaussian-distributed (mean=0, variance=$\epsilon$) additive noise in an image.
\item \textbf{Speckle Noise:} add Speckle noise (mean=0, variance=$\epsilon$) in an image. Speckle noise \citep{scikit-image} is categorized into multiplicative noise of the clean image and Gaussian noised image.
\item \textbf{Adversarial Attack:} use a pre-trained model to apply $L_{\infty}$ PGDAttack (PGDA) in an image. We adopt the default setting of ~\citep{ding2019advertorch} while replacing the hyper-parameter epsilon with our $\epsilon$.

\squishend
Each test dataset consists of 10K images. As mentioned before, we use $\epsilon$ to denote the degree of untruthful reports by referring to truthful reports (clean images). For both MNIST and CIFAR-10, we split the dataset evenly into 200 groups (200 agents each holding a group of images). For $\epsilon\in \{0.03, 0.05, 0.07, 0.1, 0.15, 0.2, 0.25, 0.3, 0.35, 0.4, 0.45$,  $ 0.5\}$, we assume each agent submits 50 images with $\epsilon-$noise to the central designer at one time. The mean and standard deviation of agents' scores are calculated concerning these 200 submissions. 

Fréchet Inception Distance (FID score) \citep{Fid} is a widely accepted measure of similarity between two datasets of images. In our experiments, the FID score is considered as a measure of the truthfulness of agents' reports by referring to images for verification.

\paragraph{Interpretation of the visualization} In Figure 1 and 2, the $x-$axis indicates the level of untruthful report---a large $\epsilon$ represents high-level untruthful (noisy) reports. The left $y-$axis is the score (given by our proposed method) of submitted $\epsilon-$level untruthful reports, while the right $y-$axis is the FID score. The curves visualize the relationship between the untruthful reports and the corresponding score/payment given by these two scoring methods. To validate the incentive property of our proposed score functions, the score of submitted reports is supposed to be monotonically decreasing w.r.t. the increasing $\epsilon-$level of untruthfulness. FID score has the incentive property if it is monotonically increasing w.r.t. the increasing $\epsilon-$level of untruthfulness.

\paragraph{With ground truth verification} In this case, we consider the test images of MNIST and CIFAR-10 as ground truth images for verification. We report the average scores/payments of all 200 agents with their standard deviation. As shown in Figure \ref{Fig:fig1}, for untruthful reports using Gaussian and Speckle noise, a larger $\epsilon$ will lead to consistently lower score/payment, establishing the incentive-compatibility of our scoring mechanism. In this case, the untruthful report also leads to a higher FID score (less similarity) - we think this is an interesting observation implying that FID can also serve as a heuristic metric for evaluating image samples when we have ground truth verification. For PGDA untruthful reports, our mechanism is robust especially when $\epsilon$ is not too large.

\paragraph{Without ground-truth verification} When we do not have access to the ground truth, we use only %
peer-reported images for verification.
Again we report the average payment with the standard deviation. As shown in Figure \ref{Fig:fig2}, FID fails to continue to be a valid measure of truthfulness when we are using peer samples (reports) for verification. However, it is clear that our mechanism is robust to peer reports for verification: truthful reports result in a higher score.

\section{Concluding Remarks}

In this work, we introduce the problem of sample elicitation as an alternative to elicit complicated distribution. Our elicitation mechanism leverages the variational form of $f$-divergence functions to achieve accurate estimation of the divergences using samples. We provide the theoretical guarantee for both our estimators and the achieved incentive compatibility. Experiments on a synthetic dataset, MNIST, and CIFAR-10 test dataset further validate incentive properties of our mechanism. It remains an interesting problem to find out more ``organic" mechanisms for sample elicitation that requires (i) less elicited samples; and (ii) induced strict truthfulness instead of approximated ones.

\section*{Acknowledgement}
This work is partially supported by the National Science Foundation (NSF) under grant IIS-2007951.

%% file: supplement.tex
\section*{Appendix}

\section{Auxiliary Results via Deep Neural Networks} \label{sec:nn-app}

\subsection{Estimation via Deep Neural Networks}\label{sec:est-dnn}

Since the most general estimator $D_f^\natural (q\|p)$ proposed in \eqref{eq:gen-form} requires solving an optimization problem over a function space, which is usually intractable, we introduce an estimator of the $f$-divergence $D_f(q\| p)$ using the family of deep neural networks in this section.    We now define the family of deep neural networks as follows. 

\begin{definition}\label{def:nn_set}
Given a vector $k = (k_0, \ldots, k_{L+1})\in\NN^{L+2}$, where $k_0 = d$ and $k_{L+1} = 1$, the family of deep neural networks is defined as
\$
\Phi(L, k) = \{ & \varphi(x; W, v) = W_{L+1}\sigma_{v_L}\cdots W_{2}\sigma_{ v_1}W_1 x\colon \notag\\
&\qquad   W_{j}\in\RR^{k_j\times k_{j-1}}, v_{j}\in\RR^{k_j} \}. 
\$
where $\sigma_{v}(x) = \max\{0, x-v\}$ is the ReLU activation function. 
\end{definition}

To avoid overfitting, the sparsity of the deep neural networks is a typical assumption  in  deep learning literature.   In practice, such a sparsity property is achieved through certain techniques, e.g., dropout \citep{srivastava2014dropout},  or certain network architecture, e.g., convolutional neural network \citep{krizhevsky2012imagenet}.  We now define the family of sparse neural networks as follows, 
\#\label{def:nn_model}
\Phi_M(L,k, s) = \Bigl\{ & \varphi(x; W, v) \in \Phi(L, d)\colon \|\varphi\|_\infty\leq M,  \|W_j\|_\infty\leq 1~\text{for~}j\in[L+1],\notag\\
& \|v_j\|_\infty  \leq 1~\text{for~}j\in[L],  ~ \sum_{j = 1}^{L+1}\|W_j\|_0 + \sum_{j = 1}^L \|v_j\|_0\leq s \Bigr\},
\#
where $s$ is the sparsity.  In contrast, another approach to avoid overfitting in deep learning literature is to control the norm of parameters \citep{li2018tighter}.  See Section \S\ref{subsection:norm_control} for details.    

Consider the following estimators via deep neural networks, 
\#\label{problem_formulation}
&\hat t(\cdot;p,q) = \argmin_{t\in\Phi_M(L,k,s)} \EE_{x\sim\PP_n}[ f^\dagger(t(x))] - \EE_{x\sim\QQ_n}[t(x)] ,\notag\\
&\hat  D_f(q\|p) =  \EE_{x\sim\QQ_n}[ \hat t(x;p,q) ] -\EE_{x\sim\PP_n}[ f^\dagger(\hat t(x;p,q) )].
\#

The following theorem characterizes the statistical rate of convergence of the estimator proposed in \eqref{problem_formulation}. 

\begin{theorem}\label{thm:main}
Let  $L = \cO(\log n)$, $s = \cO(N \log n)$, and $k = (d, d, \cO(dN), \cO(dN), \ldots, \cO(dN), 1)$ in \eqref{def:nn_model}, where $N = n^{d/(2\beta + d)}$. Under Assumptions \ref{assum:ratio}, \ref{assum:beta-holder}, and \ref{assumption:reg}, if $d < 2\beta$, then with probability at least $1- \exp\{-n^{d/(2\beta+d)}\log^5 n\}$, we have
\$
| \hat D_f(q\| p) - D_f(q\| p)| \lesssim n^{-{\beta}/{(2\beta + d)}} \log^{7/2} n. 
\$ 
\end{theorem}

We defer the proof of the theorem in Section \S \ref{proof:thm:main}. 
By Theorem \ref{thm:main}, the estimators in \eqref{problem_formulation} achieve the optimal nonparametric rate of convergence \citep{stone1982optimal} up to a logarithmic term.  We can see that by setting $\gamma_\Phi = d/\beta$ in Theorem \ref{thm:gen}, we recover the result in Theorem \ref{thm:main}. 
By \eqref{eq:eps2} and Theorem \ref{thm:main}, we have 
\$
&  \delta(n) = 1- \exp\{-n^{d/(2\beta+d)}\log^5 n\}, \qquad \epsilon(n) = c \cdot n^{-{\beta}/{(2\beta + d)}} \log^{7/2} n,
\$
where $c$ is a positive absolute constant.

\subsection{Reconstruction via Deep Neural Networks}

To utilize the estimator $\hat D_f(q\|p)$ proposed via deep neural networks in Section \S \ref{sec:est-dnn}, we propose the following estimator, 
\#
& \hat q = \argmin_{q\in \cQ}\hat D_f(q\|p), \label{eq:q_outer_sample}
\# 
where $ \hat D_f(q\|p)$ is given in \eqref{problem_formulation}. 

We impose the following assumption on the covering number of the probability density function space $\cQ$.

\begin{assumption}\label{assum:cov_num}
We have $N_2(\delta, \cQ) = \cO( \exp\{ \delta^{-d/\beta} \})$.
\end{assumption}

The following theorem characterizes the error bound of estimating $q^*$ by $\hat q$.
\begin{theorem}\label{theorem:outer_prob_sparsity} 
Under the same assumptions in Theorem \ref{thm:main}, further if Assumption \ref{assum:cov_num} holds, for sufficiently large sample size $n$, with probability at least $1-1/n$, we have
\$
D_f(\hat q\|p) & \lesssim n^{-\frac{\beta}{2\beta+d}}\cdot \log^7 n + \min_{\tilde q\in\cQ} D_f(\tilde q\|p). 
\$
\end{theorem}

The proof of Theorem \ref{theorem:outer_prob_sparsity} is deferred in Section \S\ref{subsection:proof_outer_prob_sparsity2}. We can see that by setting $\gamma_\Phi = d/\beta$ in Theorem \ref{theorem:outer_prob_sparsity}, we recover the result in Theorem \ref{theorem:outer_prob_sparsity1}.

\subsection{Auxiliary Results on Sparsity Control} 

In this section, we provide some auxiliary results on \eqref{problem_formulation}.  We first state an oracle inequality showing the rate of convergence of $\hat t(x; p,q)$.

\begin{theorem}\label{theorem:oracle_ineq}
Given $0<\varepsilon<1$, for any sample size $n$ satisfies that $n\gtrsim [\gamma + \gamma^{-1}\log(1/\varepsilon)]^2$, under Assumptions \ref{assum:ratio}, \ref{assum:beta-holder}, and \ref{assumption:reg}, it holds that
\$
\| \hat t - t^*\|_{L_2(\PP)} \lesssim  \min_{\tilde t\in\Phi_M(L,k,s)}  \|\tilde t-t^*\|_{L_2(\PP)} +   \gamma n^{-1/2}\log n  +  n^{-1/2} [\sqrt{\log(1/\varepsilon)} + \gamma^{-1}\log(1/\varepsilon) ]
\$
with probability at least $1-\varepsilon\cdot\exp(-\gamma^2)$. Here $\gamma = s^{1/2}\log(V^2 L)$ and $V=\prod_{j = 0}^{L+1}(k_j+1)$. 
\end{theorem}

We defer the proof of to Section \S \ref{section:proof_oracle_ineq}.

As a by-product, note that $t^*(x; p, q) = f'(\theta^*(x; p, q)) = f'(q(x)/p(x))$, based on the error bound established in Theorem \ref{theorem:oracle_ineq}, we obtain the following result.

\begin{corollary}
Given $0<\varepsilon<1$, for the sample size $n\gtrsim [\gamma + \gamma^{-1}\log(1/\varepsilon)]^2$, under Assumptions \ref{assum:ratio}, \ref{assum:beta-holder}, and \ref{assumption:reg}, it holds with probability at least $1-\varepsilon\cdot\exp(-\gamma^2)$ that 
\$
\| \hat \theta - \theta^*\|_{L_2(\PP)}  \lesssim \min_{\tilde t\in\Phi_M(L, k, s)}  \|\tilde t-t^*\|_{L_2(\PP)}  +   \gamma n^{-1/2}\log n  +  n^{-1/2} [\sqrt{\log(1/\varepsilon)} + \gamma^{-1}\log(1/\varepsilon) ].
\$
Here $\gamma = s^{1/2}\log(V^2 L)$ and $V=\prod_{j = 0}^{L+1}(k_j+1)$. 
\end{corollary}
\begin{proof}
Note that $(f')^{-1} = (f^\dagger)'$ and $f^\dagger$ has Lipschitz continuous gradient with parameter $1/\mu_0$ from Assumption \ref{assumption:reg} and Lemma \ref{lemma:strong_convex}, we obtain the result from Theorem \ref{theorem:oracle_ineq}. 
\end{proof}

\subsection{Error Bound using Norm Control}\label{subsection:norm_control}
In this section, we consider using norm of the parameters (specifically speaking, the norm of $W_j$ and $v_j$ in \eqref{def:nn_set}) to control the error bound, which is an alternative of the network model shown in \eqref{def:nn_model}. We consider the family of $L$-layer neural networks with bounded spectral norm for weight matrices $W = \{W_j\in \RR^{k_j\times k_{j-1}}\}_{j = 1}^{L+1}$, where $k_0 = d$ and $k_{L+1} = 1$, and vector $v = \{v_j\in\RR^{k_j}\}_{j = 1}^L$, which is denoted as
\#\label{eq:nn_norm}
&\Phi_{\text{norm}} = \Phi_{\text{norm}}(L,k,A,B) =  \{\varphi(x; W, v) \in\Phi(L,k): \|v_j\|_2\leq A_j ~\text{for~all~} j\in[L], \\
&\qquad\qquad\qquad\qquad\qquad\qquad\qquad\qquad\qquad\qquad\qquad\qquad\qquad \|W_j\|_2\leq B_j ~\text{for~all~} j\in[L+1] \},\notag
\#
where $\sigma_{v_j}(x) = \max\{0, x - v_j\}$ for any $j\in[L]$. We write the following optimization problem,
\#\label{eq:problem_norm_control}
&\hat t(x;p,q) = \argmin_{t\in\Phi_{\text{norm}}} \EE_{x\sim\PP_n} [ f^\dagger (t(x) ) ] - \EE_{x\sim\QQ_n} [t(x) ] ,\notag\\
&\hat  D_f(q\|p) =  \EE_{x\sim \QQ_n} [ \hat t(x;p,q)  ] -\EE_{x\sim \PP_n} [ f^\dagger ( \hat t(x;p,q)  ) ].
\#
Based on this formulation, we derive the error bound on the estimated $f$-divergence in the following theorem. We only consider the generalization error in this setting. Therefore, we assume that the ground truth $t^*(x; p,q) = f'(q(x)/p(x))\in \Phi_{\text{norm}}$.  Before we state the theorem, we first define two parameters for the family of neural networks $\Phi_{\text{norm}}(L,k,A,B)$ as follows, 
\#\label{eq:gamma_def}
\gamma_1 = B \prod_{j = 1}^{L+1} B_j \cdot \sqrt{\sum_{j = 0}^{L+1}k_j^2},\qquad \gamma_2 = \frac{L\cdot  (\sqrt{\sum_{j = 1}^{L+1}k_j^2B_j^2} + \sum_{j = 1}^L A_j )}{\sum_{j = 0}^{L+1}k_j^2\cdot \min_{j} B_{j}^2}\cdot \sum_{j = 1}^L A_j.
\#
Now, we  state the theorem. 

\begin{theorem}\label{theorem:norm_control_theorem}
We assume that $t^*(x; p, q)\in\Phi_{\text{norm}}$. Then for any $0<\varepsilon<1$, with probability at least $1-\varepsilon$, it holds that
\$
 |\hat D_f(q\|p) - D_f(q\|p) |\lesssim \gamma_1\cdot n^{-1/2}\log(\gamma_2 n) + \prod_{j = 1}^{L+1}B_j\cdot n^{-1/2}\sqrt{{\log(1/\varepsilon)}}, 
\$
where $\gamma_1$ and $\gamma_2$ are defined in \eqref{eq:gamma_def}. 
\end{theorem}

We defer the proof to Section \S \ref{subsection:proof_norm_control_theorem}.

The next theorem characterizes the rate of convergence of $\hat q = \argmin_{q\in \cQ} \hat D_f(q\| p)$, where $\hat D_f(q\| p)$ is proposed in \eqref{eq:problem_norm_control}. 

\begin{theorem}\label{theorem:outer_prob_norm}
For any $0<\varepsilon < 1$, with probability at least $1-\varepsilon$, we have
\$
D_f(\hat q\|p) \lesssim b_2(n,\gamma_1, \gamma_2) + \prod_{j = 1}^{L+1}B_j\cdot n^{-1/2}\cdot \sqrt{ \log ( N_2 [b_2(n,\gamma_1, \gamma_2), \cQ ]/\varepsilon  )} +  \min_{\tilde q\in\cQ}D_f(\tilde q\|p),
\$
where $b_2(n,\gamma_1, \gamma_2) = \gamma_1 n^{-1/2}\log (\gamma_2 n)$, and $N_2(\delta, \cQ)$ is the covering number of $\cQ$. 
\end{theorem}

We defer the proof to Section \S \ref{subsection:proof_outer_prob_norm}.

\section{Exemplary $\hat{\tran}$ and $f^{\dag}$}

As for experiments on MNIST and CIFAR-10, we choose to skip Step 1 in Algorithm \ref{f1} and \ref{f2} and instead adopt $\hat{\tran}$ and $f^{\dag}$ as suggested by \citep{nowozin2016f}. Exemplary $\hat{\tran}$ and $f^{\dag}$ are specified in Table 2.
\begin{table}[H]
\small
\caption{Exemplary $\hat{\tran}$, $f^{\dag}$.}
\begin{center}
\begin{tabular}{ l l l l l} 
 \hline
 Name & $D_f(\PP||\QQ)$ & $\hat{\tran}(v)$ & $\text{dom}_{f^{\dag}}$ & $f^{\dag}(u)$ \\ 
 \hline
 Total Variation & $\int \dfrac{1}{2}|p(z)-q(z)|dz$ & $\dfrac{1}{2}\tanh(v)$  & $u\in [-\dfrac{1}{2}, \dfrac{1}{2}]$ & $u$\\
 Jenson-Shannon & $\int p(x)\log{\dfrac{p(z)}{q(z)}}$ & $\log{\dfrac{2}{1+e^{-v}}}$  & $u<\log{2}$  & $-\log{(2-e^{u})}$\\
 Squared Hellinger  &$\int (\sqrt{p(z)}-\sqrt{q(z)})^2 dz$ &$1-e^{v}$  & $u<1$ & $\dfrac{u}{1-u}$\\
 Pearson $\mathbf{x}$  & $\int \dfrac{(q(z)-p(z))^2}{p(z)} dz$ & $v$ &  $\mathbb{R}$ & $\dfrac{1}{4}u^2+u$ \\
 Neyman $\mathbf{x}$  & $\int \dfrac{(p(z)-q(z))^2}{p(z)} dz$ &$1-e^{v}$ & $u<1$ & $2-2\sqrt{1-u}$\\
 KL  & $\int p(z) \log{\dfrac{p(z)}{q(z)}} dx$ & $v$  &$\mathbb{R}$ & $e^{u-1}$\\
 Reverse KL  & $\int q(z) \log{\dfrac{q(z)}{p(z)}} dz$ & $-e^{v}$ & $\mathbb{R}_-$  &$-1-\log{(-u)}$ \\
 Jeffrey  & $\int (q(z)-p(z)) \log{\dfrac{p(z)}{q(z)}} dz$ &$v$ &  $\mathbb{R}$ & $W(e^{1-u})+\dfrac{1}{W(e^{1-u})}+u-2$\\
\hline
\end{tabular}
\end{center}
\label{Tab:f_div}
\end{table}

\section{Proofs of Theorems}

\subsection{Proof of Theorem \ref{thm:multi}}\label{proof:thm:multi}

If the player truthfully reports, she will receive the following expected payment per sample $i$: with probability at least $1-\delta(n)$,
\begin{align*}
\E[S(r_i, \cdot )] &:=a - b ( \E_{x \sim \QQ_n} [\hat{\tran}(x)] - \E_{x_i \sim \PP_n}[f^{\dag} (\hat{\tran}(x_i))] )\\
&=a - b \cdot \hat{D}_f(q\|p) \\
&\geq a - b\cdot (D_f(q \| p)+\epsilon(n)) ~~\text{(sample complexity guarantee)}\\
&\geq a - b\cdot (D_f(p \| p)+\epsilon(n)) ~~\text{(agent believes $p = q$)}\\
&= a - b\epsilon(n)
\end{align*}
Similarly, any misreporting according to a distribution $\tilde{p}$ with distribution $\tilde{\PP}$ will lead to the following derivation with probability at least $1-\delta$
\begin{align*}
\E[S(r_i, \cdot )] &:=a - b ( \E_{x \sim \QQ_n} [\hat{\tran}(x)] - \E_{x_i \sim \tilde{\PP}_n}[f^{\dag} (\hat{\tran}(x_i))] )\\
&=a - b \cdot \hat{D}_f(q\|\tilde{p}) \\
&\leq a - b\cdot (D_f(p\|\tilde{p})-\epsilon(n)) \\
&\leq a + b\epsilon(n) ~~ (\text{non-negativity of $D_f$})
\end{align*}
Combining above, and using union bound, leads to $(2\delta(n), 2b \epsilon(n) )$-properness.

\subsection{Proof of Theorem \ref{thm:single}}\label{proof:thm:single}
Consider an arbitrary agent $i$. Suppose every other agent truthfully reports. 
\begin{align*}
\E [S(r_i, \{r_j\}_{j \neq i}) ] &=a + b ( \E_{\xb \sim \PP_n \oplus \QQ_n|r_i} [\hat{\tran}(\xb)] - \E_{\xb \sim \PP_n \times \QQ_n|r_i}\{ f^{\dag} (\hat{\tran}(\xb))\} ) \\\
&= a + b \E [ \E_{\xb \sim \PP_n \oplus \QQ_n|r_i} [\hat{\tran}(x)] - \E_{\xb \sim \PP_n \times \QQ_n|r_i}\{ f^{\dag} (\hat{\tran}(\xb))\} ]
\end{align*}
Consider the divergence term $\E [\E_{\xb \sim \PP_n \oplus \QQ_n|r_i} [\hat{\tran}(x)] - \E_{\xb \sim \PP_n \times \QQ_n|r_i}\{ f^{\dag} (\hat{\tran}(\xb))\} ]$. Reporting a $r_i \sim \tilde{\PP} \neq \PP$ (denote its distribution as $\tilde{p}$) leads to the following score
\begin{align*}
&~~~~\E_{r_i \sim \tilde{\PP}_n}[\E_{\xb \sim \tilde{\PP}_n \oplus \QQ_n|r_i} [\hat{\tran}(\xb)] - \E_{\xb \sim \tilde{\PP}_n \times \QQ_n|r_i}\{ f^{\dag} (\hat{\tran}(\xb))\}] \\
&=  \E_{\xb \sim \tilde{\PP}_n \oplus \QQ_n} [\hat{\tran}(\xb)] - \E_{\xb \sim \tilde{\PP}_n \times \QQ_n}\{ f^{\dag} (\hat{\tran}(\xb))\}~~\text{(tower property)}\\
&\leq \max_{\tran}\E_{\xb \sim \tilde{\PP}_n\oplus \QQ_n} [\tran(\xb)] - \E_{\xb \sim \tilde{\PP}_n \times \QQ_n}\{ f^{\dag} (\tran(\xb))\}~~\text{(max)}
\\
&=\hat{D}_f(\tilde{p}\oplus q \|\tilde{p}\times  q)
\\
&\leq D_f(\tilde{p}\oplus q \|\tilde{p}\times  q)+ \epsilon(n)\\
&=I_f(\tilde{p};q)+ \epsilon(n) ~~\text{(definition)}\\
&\leq I_f(p;q) + \epsilon(n)~~\text{(data processing inequality \citep{kong2019information})}\\
\end{align*}
with probability at least $1-\delta(n)$ (the other $\delta(n)$ probability with maximum score $\bar{S}$).

Now we prove that truthful reporting leads at least
\[
 I_f(p;q)- \epsilon(n)
\]
of the divergence term:
\begin{align*}
&~~~~\E_{x_i \sim \PP_n}[\E_{\xb \sim \PP_n \oplus \QQ_n|x_i} [\hat{\tran}(\xb)] - \E_{\xb \sim \PP_n \times \QQ_n|x_i}\{ f^{\dag} (\hat{\tran}(\xb))\}] \\
&=   \E_{\xb \sim \PP_n \oplus \QQ_n} [\hat{\tran}(\xb)] - \E_{\xb \sim \PP_n \times \QQ_n}\{ f^{\dag} (\hat{\tran}(\xb))\} ~~\text{(tower property)}\\
&=\hat{D}_f(p\oplus q \|p\times  q)\\
&\geq D_f(p\oplus q \|p \times  q) - \epsilon(n) \\
&=I_f(p;q)  - \epsilon(n) ~~\text{(definition)}
\end{align*}
with probability at least $1-\delta(n)$ (the other $\delta(n)$ probability with score at least 0). Therefore the expected divergence terms differ at most by $2\epsilon(n)$ with probability at least $1-2\delta(n)$ (via union bound). The above combines to establish a $(2\delta(n), 2b\epsilon(n))$-BNE.

\subsection{Proof of Theorem \ref{thm:gen}} \label{proof:thm:gen}
We first show the convergence of $t^\natural$, and then the convergence of $D^\natural_f(q\| p)$.   For any real-valued function $\varrho$, we write $\EE_{\PP}(\varrho) = \EE_{x\sim \PP}[\varrho(x)]$, $\EE_{\QQ}(\varrho) = \EE_{x\sim \QQ}[\varrho(x)]$, $\EE_{\PP_n}(\varrho) = \EE_{x\sim \PP_n}[\varrho(x)]$, and $\EE_{\QQ_n}(\varrho) = \EE_{x\sim \QQ_n}[\varrho(x)]$ for notational convenience. 

For any $\tilde t \in \Phi$,  we establish the following lemma.  

\begin{lemma}\label{lemma:ineq1}
Under the assumptions stated in Theorem \ref{thm:gen}, it holds that
\$
{1}/(4L_0)\cdot \| t^\natural - t^*\|_{L_2(\PP)}^2  & \leq    \{\EE_{\QQ_n} [(t^\natural - t^*)/2 ] - \EE_{\QQ} [(t^\natural - t^*)/2 ]  \} \\
&\qquad  -  \{\EE_{\PP_n} [f^\dagger ((t^\natural + t^*)/2 ) - f^\dagger(t^*)  ] - \EE_{\PP} [f^\dagger ((t^\natural + t^*)/2 ) - f^\dagger(t^*)  ] \}. 
\$
Here  $\mu_0$ and $L_0$ are specified in Assumption \ref{assumption:reg}. 
\end{lemma}

We defer the proof to Section \S \ref{subsection:proof_ineq1}. 

Note that by Lemma \ref{lemma:ineq1} and the fact that $f^\dagger$ is Lipschitz continuous, we have
\#\label{eq:bound-add11}
\|t^\natural-t^*\|_{L_2(\PP)}^2  & \lesssim  \{\EE_{\QQ_n} [(t^\natural - t^*)/2 ] - \EE_{\QQ} [(t^\natural - t^*)/2 ]  \} \notag \\
&\qquad-  \{\EE_{\PP_n} [f^\dagger ((t^\natural + t^*)/2 ) - f^\dagger(t^*)  ] - \EE_{\PP} [f^\dagger ((t^\natural + t^*)/2 ) - f^\dagger(t^*)  ] \}. 
\#
Further, to upper bound the RHS of \eqref{eq:bound-add1}, we establish the following lemma. 

\begin{lemma}\label{lemma:fake_vdg2}
We assume that the function $\psi:\RR\to\RR$ is Lipschitz continuous and bounded such that $|\psi(x)|\leq M_0$ for any $|x|\leq M$. Then under the assumptions stated in Theorem \ref{theorem:oracle_ineq}, we have 
\$
& \PP \biggl\{\sup_{t\colon \psi(t)\in \Psi}  \frac{ |\EE_{\PP_n} [\psi(t) - \psi(t^*) ]   - \EE_{\PP} [\psi(t) -  \psi(t^*) ]     |}{  n^{-2/(\gamma_\Phi + 2)} } \geq c_2  \biggr\} \leq c_1 \exp(-n^{\gamma_\Phi / (2 + \gamma_\Phi)} / c_1^2),
\$
where $c_1$ and $c_2$ are positive absolute constants. 
\end{lemma}

We defer the proof to Section \S \ref{subsection:proof_fake_vdg2}.

Note that the results in Lemma \ref{lemma:fake_vdg2} also apply to the distribution $\QQ$, and by using the fact that the true density ratio $\theta^*(x; p, q) = q(x)/p(x)$ is bounded below and above, we know that $L_2(\QQ)$ is indeed equivalent to $L_2(\PP)$. We thus focus on $L_2(\PP)$ here. 
By \eqref{eq:bound-add11},  Lemma \ref{lemma:fake_vdg2}, and the Lipschitz property of $f^\dagger$ according to Lemma \ref{lemma:strong_convex}, with probability at least $1-c_1 \exp(-n^{\gamma_\Phi / (2 + \gamma_\Phi)} / c_1^2)$, we have 
\#\label{eq:jan1}
&\|  t^\natural - t^*\|_{L_2(\PP)} \lesssim  n^{-1/(\gamma_\Phi + 2)}. 
\#

Note that we have
\#\label{eq:all-bs}
& | D_f^\natural(q\| p) - D_f(q\| p)| \notag\\
& \qquad \leq |  \EE_{\QQ_n}[t^\natural - t^*] - \EE_{\QQ}[t^\natural - t^*]      | + |  \EE_{\PP_n}[f^\dagger(t^\natural) - f^\dagger(t^*)] - \EE_{\PP}[f^\dagger(t^\natural) - f^\dagger(t^*)]      |\notag\\
& \qquad \qquad + | \EE_\QQ[t^\natural - t^*] - \EE_\PP[  f^\dagger(t^\natural) - f^\dagger (t^*) ]  | + |   \EE_{\QQ_n}[t^*] - \EE_\QQ[t^*]    | +  | \EE_{\PP_n}[f^\dagger(t^*)]  - \EE_{\PP}[f^\dagger(t^*)]  |\notag\\
& \qquad = B_1 + B_2 + B_3 + B_4 + B_5. 
\#
We upper bound $B_1$, $B_2$, $B_3$, $B_4$, and $B_5$ in the sequel.   First, by Lemma \ref{lemma:fake_vdg2}, with probability at least $1 - c_1 \exp(-n^{\gamma_\Phi / (2 + \gamma_\Phi)} / c_1^2)$, we have
\#\label{eq:b1-bound}
B_1 \lesssim n^{-2 / (\gamma_\Phi + 2)}. 
\#
Similar upper bound also holds for $B_2$.  Also, following from \eqref{eq:jan1}, with probability at least $1-c_1 \exp(-n^{\gamma_\Phi / (2 + \gamma_\Phi)} / c_1^2)$, we have
\#\label{eq:b3-bound}
B_3 \lesssim n^{-1/(\gamma_\Phi + 2)}. 
\# 
Meanwhile, by Hoeffding’s inequality, with probability at least $1-c_1 \exp(-n^{\gamma_\Phi / (2 + \gamma_\Phi)} / c_1^2)$, we have
\#\label{eq:b4-bound}
B_4 \lesssim n^{-1/(\gamma_\Phi + 2)}. 
\# 
Similar upper bound also holds for $B_5$.  Now, combining \eqref{eq:all-bs}, \eqref{eq:b1-bound}, \eqref{eq:b3-bound}, and \eqref{eq:b4-bound}, with probability at least $1-c_1 \exp(-n^{\gamma_\Phi / (2 + \gamma_\Phi)} / c_1^2)$, we have
\$
| D_f^\natural(q\| p) - D_f(q\| p)| \lesssim n^{-1/(\gamma_\Phi + 2)}. 
\$
We conclude the proof of Theorem \ref{thm:gen}.

\subsection{Proof of Theorem \ref{thm:main}}\label{proof:thm:main}
\textbf{Step 1.}  We upper bound  $\|t^* - \hat t\|_{L_2(\PP)}$ in the sequel. Note that $t^* \in \Omega\subset[a,b]^d$.  
To invoke Theorem \ref{theorem:approx}, we denote by $t'(y) = t^*((b-a)y+a\mathbf{1}_d)$, where $\mathbf{1}_d = (1, 1, \ldots, 1)^\top\in\RR^d$. Then the support of $t'$ lies in the unit cube $[0,1]^d$. We choose $L' = \cO(\log n), s' = \cO(N \log n), k' = (d, \cO(dN), \cO(dN), \ldots, \cO(dN), 1)$, and $m' = \log n$, we then utilize Theorem \ref{theorem:approx} to construct some $ \tilde t'\in \Phi_M(L', k', s')$ such that
\$
\|  \tilde t'  - t'\|_{L_\infty ([0,1]^d)}\lesssim N^{-\beta/d}.
\$
We further define $\tilde t(\cdot) = \tilde t'\circ \ell(\cdot)$, where $\ell(\cdot)$ is a linear mapping taking the following form, 
\$
\ell (x) = \frac{x}{b-a} - \frac{a}{b-a}\cdot \mathbf{1}_d.
\$
To this end, we know that $\tilde t\in\Phi_M(L, k, s)$, with parameters $L$, $k$, and $s$ given in the statement of Theorem \ref{thm:main}. 
We fix this $\tilde t$ and invoke Theorem \ref{theorem:oracle_ineq}, then with probability at least $1-\varepsilon\cdot\exp(-\gamma^2)$, we have
\#\label{eq:fff}
\| \hat t - t^*\|_{L_2(\PP)} & \lesssim  \|\tilde t-t^*\|_{L_2(\PP)} +   \gamma n^{-1/2}\log n  +  n^{-1/2} [\sqrt{\log(1/\varepsilon)} + \gamma^{-1}\log(1/\varepsilon) ] \notag\\ 
& \lesssim  N^{-\beta/d} + \gamma n^{-1/2}\log n  +  n^{-1/2} [\sqrt{\log(1/\varepsilon)} + \gamma^{-1}\log(1/\varepsilon) ]. 
\#
Note that $\gamma$ takes the form $\gamma = s^{1/2}\log(V^2 L)$, where $V = \cO(d^L\cdot N^L)$ and $L, s$ given in the statement of Theorem \ref{thm:main}, it holds that $\gamma = \cO(N^{1/2}\log^{5/2} n)$. Moreover, by the choice $N = n^{d/(2\beta + d)}$, combining \eqref{eq:fff} and taking $\varepsilon =  1/n$, we know that 
\#\label{eq:bound-t111}
\| \hat t - t^*\|_{L_2(\PP)}  \lesssim n^{-{\beta}/{(2\beta + d)}} \log^{7/2} n
\# 
with probability at least $1- \exp\{-n^{d/(2\beta+d)}\log^5 n\}$. 

\textbf{Step 2.} Note that we have
\#\label{eq:all-bs2}
& | \hat D_f(q\| p) - D_f(q\| p)| \notag\\
& \qquad \leq |  \EE_{\QQ_n}[\hat t - t^*] - \EE_{\QQ}[\hat t - t^*]      | + |  \EE_{\PP_n}[f^\dagger(\hat t) - f^\dagger(t^*)] - \EE_{\PP}[f^\dagger(\hat t) - f^\dagger(t^*)]      |\notag\\
& \qquad \qquad + | \EE_\QQ[\hat t - t^*] - \EE_\PP[  f^\dagger(\hat t) - f^\dagger (t^*) ]  | + |   \EE_{\QQ_n}[t^*] - \EE_\QQ[t^*]    | +  | \EE_{\PP_n}[f^\dagger(t^*)]  - \EE_{\PP}[f^\dagger(t^*)]  |\notag\\
& \qquad = B_1 + B_2 + B_3 + B_4 + B_5. 
\#
We upper bound $B_1$, $B_2$, $B_3$, $B_4$, and $B_5$ in the sequel.   First, by Lemma \ref{lemma:fake_vdg}, with probability at least $1-\exp\{-n^{d/(2\beta+d)}\log^5 n\}$, we have
\#\label{eq:b1-bound2}
B_1 \lesssim n^{-{2\beta}/{(2\beta + d)}} \log^{7/2} n. 
\#
Similar upper bound also holds for $B_2$.  Also, following from \eqref{eq:bound-t111}, with probability at least $1- \exp\{-n^{d/(2\beta+d)}\log^5 n\}$, we have
\#\label{eq:b3-bound2}
B_3 \lesssim n^{-{\beta}/{(2\beta + d)}} \log^{7/2} n. 
\# 
Meanwhile, by Hoeffding’s inequality, with probability at least $1 - \exp(-n^{ d / (2\beta + d)})$, we have
\#\label{eq:b4-bound2}
B_4 \lesssim n^{-{\beta}/{(2\beta + d)}}. 
\# 
Similar upper bound also holds for $B_5$.  Now, combining \eqref{eq:all-bs2}, \eqref{eq:b1-bound2}, \eqref{eq:b3-bound2}, and \eqref{eq:b4-bound2}, with probability at least $1- \exp\{-n^{d/(2\beta+d)}\log^5 n\}$, we have
\$
| \hat D_f(q\| p) - D_f(q\| p)| \lesssim n^{-{\beta}/{(2\beta + d)}} \log^{7/2} n. 
\$
We conclude the proof of Theorem \ref{thm:main}.

\subsection{Proof of Theorem \ref{theorem:outer_prob_sparsity1}}\label{subsection:proof_outer_prob_sparsity1}
We first need to bound the max deviation of the estimated $f$-divergence $D_f^\natural (q\|p)$ among all $q\in\cQ$. The following lemma provides such a bound.  

\begin{lemma}\label{lemma:for_all_hold_sparsity}
Under the assumptions stated in Theorem \ref{theorem:outer_prob_sparsity}, for any fixed density $p$, if the sample size $n$ is sufficiently large, it holds that 
\$
\sup_{q\in \cQ} |D_f(q\|p) -  D_f^\natural (q\|p) | \lesssim  n^{-{1}/{(\gamma_\Phi+2)}} \cdot \log n  
\$
with probability at least $1- 1/n$.
\end{lemma}
We defer the proof to Section \S \ref{subsection:proof_for_all_hold_sparsity}. 

Now we turn to the proof of the theorem.  We denote by $\tilde q' = \argmin_{\tilde q\in \cQ} D_f(\tilde q\| p)$, then with probability at least $1- 1/n$, we have
\#\label{eq:pp21}
D_f( q^\natural\|p) & \leq  |D_f( q^\natural\|p) -  D_f^\natural(q^\natural\|p) | + D_f^\natural(q^\natural\|p)\notag\\
&\leq \sup_{q\in \cQ} |D_f(q\|p) - D_f^\natural(q\|p) | + D_f^\natural (\tilde q'\|p) \notag\\
& \leq \sup_{q\in \cQ} |D_f(q\|p) - D_f^\natural(q\|p) | + |D_f^\natural (\tilde q'\|p) - D_f(\tilde q'\|p) | + D_f(\tilde q'\|p)\notag\\
& \lesssim n^{-{1}/{(\gamma_\Phi+2)}} \cdot \log n    + D_f(\tilde q'\|p). 
\#
Here in the second inequality we use the optimality of $ q^\natural$ over $\tilde q'\in \cQ$ to the problem \eqref{eq:q1-est}, while the last inequality uses Lemma \ref{lemma:for_all_hold_sparsity} and Theorem \ref{thm:gen}. Moreover, note that  $D_f(\tilde q'\|p) = \min_{\tilde q\in\cQ}D_f(\tilde q\|p)$, combining \eqref{eq:pp21}, it holds that with probability at least $1-1/n$, 
\$
D_f( q^\natural \|p)  \lesssim n^{-{1}/{(\gamma_\Phi+2)}} \cdot \log n    + \min_{\tilde q\in\cQ}D_f(\tilde q\|p).
\$
This concludes the proof of the theorem.

\subsection{Proof of Theorem \ref{theorem:outer_prob_sparsity}}\label{subsection:proof_outer_prob_sparsity2}
We first need to bound the max deviation of the estimated $f$-divergence $\hat D_f(q\|p)$ among all $q\in\cQ$. The following lemma provides such a bound.  

\begin{lemma}\label{lemma:for_all_hold_sparsity2}
Under the assumptions stated in Theorem \ref{theorem:outer_prob_sparsity}, for any fixed density $p$, if the sample size $n$ is sufficiently large, it holds that 
\$
\sup_{q\in \cQ} |D_f(q\|p) - \hat D_f(q\|p) | \lesssim  n^{-{\beta}/{(d+2\beta)}} \cdot \log^7 n
\$
with probability at least $1- 1/n$.
\end{lemma}
We defer the proof to Section \S \ref{subsection:proof_for_all_hold_sparsity2}. 

Now we turn to the proof of the theorem.  We denote by $\tilde q' = \argmin_{\tilde q\in \cQ} D_f(\tilde q\| p)$, then with probability at least $1- 1/n$, we have
\#\label{eq:pp2}
D_f(\hat q\|p) & \leq  |D_f(\hat q\|p) - \hat D_f(\hat q\|p) | + \hat D_f(\hat q\|p)\notag\\
&\leq \sup_{q\in \cQ} |D_f(q\|p) - \hat D_f(q\|p) | + \hat D_f(\tilde q'\|p) \notag\\
& \leq \sup_{q\in \cQ} |D_f(q\|p) - \hat D_f(q\|p) | + |\hat D_f(\tilde q'\|p) -D_f(\tilde q'\|p) | + D_f(\tilde q'\|p)  \notag\\ 
& \lesssim n^{-{\beta}/{(d+2\beta)}} \cdot \log^7 n    + D_f(\tilde q'\|p). 
\#
Here in the second inequality we use the optimality of $\hat q$ over $\tilde q'\in \cQ$ to the problem \eqref{eq:q_outer_sample}, while the last inequality uses Lemma \ref{lemma:for_all_hold_sparsity2} and Theorem \ref{thm:main}. Moreover, note that  $D_f(\tilde q'\|p) = \min_{\tilde q\in\cQ}D_f(\tilde q\|p)$, combining \eqref{eq:pp2}, it holds that with probability at least $1-1/n$, 
\$
D_f(\hat q\|p)  \lesssim  n^{-{\beta}/{(d+2\beta)}} \cdot \log^7 n    + \min_{\tilde q\in\cQ}D_f(\tilde q\|p).
\$
This concludes the proof of the theorem.

\subsection{Proof of Theorem \ref{theorem:oracle_ineq}}\label{section:proof_oracle_ineq}
For any real-valued function $\varrho$, we write $\EE_{\PP}(\varrho) = \EE_{x\sim \PP}[\varrho(x)]$, $\EE_{\QQ}(\varrho) = \EE_{x\sim \QQ}[\varrho(x)]$, $\EE_{\PP_n}(\varrho) = \EE_{x\sim \PP_n}[\varrho(x)]$, and $\EE_{\QQ_n}(\varrho) = \EE_{x\sim \QQ_n}[\varrho(x)]$ for notational convenience. 

For any $\tilde t \in \Phi_M(L, k, s)$,  we establish the following lemma.  

\begin{lemma}\label{lemma:ineq2}
Under the assumptions stated in Theorem \ref{theorem:oracle_ineq}, it holds that
\$
{1}/(4L_0)\cdot \|\hat t-\tilde t\|_{L_2(\PP)}^2  & \leq  {1}/{\mu_0}\cdot  \|\hat t - \tilde t\|_{L_2(\PP)}\cdot \|\tilde t - t^*\|_{L_2(\PP)} +   \{\EE_{\QQ_n} [(\hat t - \tilde t)/2 ] - \EE_{\QQ} [(\hat t - \tilde t)/2 ]  \} \\
&\qquad  -  \{\EE_{\PP_n} [f^\dagger ((\hat t + \tilde t)/2 ) - f^\dagger(\tilde t)  ] - \EE_{\PP} [f^\dagger ((\hat t + \tilde t)/2 ) - f^\dagger(\tilde t)  ] \}
\$
Here  $\mu_0$ and $L_0$ are specified in Assumption \ref{assumption:reg}. 
\end{lemma}

The proof of Lemma \ref{lemma:ineq2} is deferred to Section \S \ref{subsection:proof_ineq2}.

Note that by Lemma \ref{lemma:ineq2} and the fact that $f^\dagger$ is Lipschitz continuous, we have
\#\label{eq:bound-add1}
\|\hat t-\tilde t\|_{L_2(\PP)}^2  & \lesssim  \|\hat t - \tilde t\|_{L_2(\PP)}\cdot \|\tilde t - t^*\|_{L_2(\PP)} +   \{\EE_{\QQ_n} [(\hat t - \tilde t)/2 ] - \EE_{\QQ} [(\hat t - \tilde t)/2 ]  \} \notag \\
&\qquad-  \{\EE_{\PP_n} [f^\dagger ((\hat t + \tilde t)/2 ) - f^\dagger(\tilde t)  ] - \EE_{\PP} [f^\dagger ((\hat t + \tilde t)/2 ) - f^\dagger(\tilde t)  ] \}. 
\#
Furthermore, to bound the RHS of the above inequality, we establish the following lemma. 

\begin{lemma}\label{lemma:fake_vdg}
We assume that the function $\psi:\RR\to\RR$ is Lipschitz continuous and bounded such that $|\psi(x)|\leq M_0$ for any $|x|\leq M$. Then under the assumptions stated in Theorem \ref{theorem:oracle_ineq}, for any fixed $\tilde t(x)\in \Phi_M$, $n\gtrsim [\gamma + \gamma^{-1}\log(1/\varepsilon)]^2$ and $0<\varepsilon<1$, we have the follows
\$
\PP \biggl\{\sup_{t(\cdot)\in\Phi_M(L, k, s)}\frac{ |\EE_{\PP_n} [\psi(t) - \psi(\tilde t) ]   - \EE_{\PP} [\psi(t) -  \psi(\tilde t) ]   |}{\max\{\eta(n,\gamma,\varepsilon) \cdot  \|\psi(t)-\psi(\tilde t) \|_{L_2(\PP)} ,  \lambda(n,\gamma,\varepsilon)  \} }\leq 16M_0  \biggr\}\geq 1-\varepsilon\cdot\exp(-\gamma^2),
\$
where $\eta(n,\gamma,\varepsilon) = n^{-1/2}[ \gamma\log n + \gamma^{-1}\log(1/\varepsilon)   ]$ and $\lambda(n,\gamma,\varepsilon) = n^{-1} [ \gamma^2 + \log(1/\varepsilon)]$.  Here $\gamma$ takes the form $\gamma = s^{1/2}\log(V^2L)$, where $V = \prod_{j = 0}^{L+1}(k_j+1)$. 
\end{lemma}

We defer the proof to Section \S \ref{subsection:proof_fake_vdg}. 

Note that the results in Lemma \ref{lemma:fake_vdg} also apply to the distribution $\QQ$, and by using the fact that the true density ratio $\theta^*(x; p, q) = q(x)/p(x)$ is bounded below and above, we know that $L_2(\QQ)$ is indeed equivalent to $L_2(\PP)$. We thus focus on $L_2(\PP)$ here. 
By \eqref{eq:bound-add1},  Lemma \ref{lemma:fake_vdg}, and the Lipschitz property of $f^\dagger$ according to Lemma \ref{lemma:strong_convex}, with probability at least $1-\varepsilon\cdot\exp(-\gamma^2)$, we have the following bound
\#\label{eq:variance}
&\| \hat t - \tilde t\|^2_{L_2(\PP)} \lesssim  \| \hat t - \tilde t\|_{L_2(\PP)}\cdot \| \tilde t - t^*\|_{L_2(\PP)}\notag \\
&\qquad + \cO (n^{-1/2} [ \gamma\log n + \gamma^{-1}\log(1/\varepsilon)    ] \cdot \|\hat t-\tilde t\|_{L_2(\PP)}    \vee  n^{-1}[ \gamma^2 + \log(1/\varepsilon)]   ),
\#
where we recall that the notation $\gamma = s^{1/2}\log(V^2L)$ is a parameter related with the family of neural networks $\Phi_M$. We proceed to analyze the dominant part on the RHS of \eqref{eq:variance}.

\vskip5pt
\noindent\textbf{Case 1.} If the term $\| \hat t - \tilde t\|_{L_2(\PP)}\cdot \| \tilde t - t^*\|_{L_2(\PP)}$ dominates, then with probability at least $1-\varepsilon\cdot\exp(-\gamma^2)$
\$
\| \hat t - \tilde t\|_{L_2(\PP)}\lesssim \| \tilde t - t^*\|_{L_2(\PP)}.
\$

\vskip5pt
\noindent\textbf{Case 2.} If the term $\cO(n^{-1/2} [ \gamma\log n + \gamma^{-1}\log(1/\varepsilon)    ] \cdot \|\hat t-\tilde t\|_{L_2(\PP)})$ dominates, then with probability at least $1-\varepsilon\cdot\exp(-\gamma^2)$
\$
\| \hat t - \tilde t\|_{L_2(\PP)}\lesssim n^{-1/2} [ \gamma\log n + \gamma^{-1}\log(1/\varepsilon)    ].
\$

\vskip5pt
\noindent\textbf{Case 3.} If the term $\cO (n^{-1}[ \gamma^2 + \log(1/\varepsilon)] )$ dominates, then with probability at least $1-\varepsilon\cdot\exp(-\gamma^2)$
\$
\| \hat t - \tilde t\|_{L_2(\PP)}\lesssim n^{-1/2} [\gamma + \sqrt{\log(1/\varepsilon)} ].
\$

Therefore, by combining the above three cases, we have 
\$
\| \hat t -\tilde t\|_{L_2(\PP)}  \lesssim  \|\tilde t-t^*\|_{L_2(\PP)} + \gamma n^{-1/2}\log n  +  n^{-1/2} [\sqrt{\log(1/\varepsilon)} + \gamma^{-1}\log(1/\varepsilon) ]. 
\$
Further combining the triangle inequality, we have
\#\label{eq:at1}
\| \hat t - t^*\|_{L_2(\PP)} \lesssim  \|\tilde t-t^*\|_{L_2(\PP)} + \gamma n^{-1/2}\log n  +  n^{-1/2} [\sqrt{\log(1/\varepsilon)} + \gamma^{-1}\log(1/\varepsilon) ]
\#
with probability at least $1-\varepsilon\cdot\exp(-\gamma^2)$. Note that \eqref{eq:at1} holds for any $\tilde t\in\Phi_M(L,k,s)$, especially for the choice $\tilde t$ which  minimizes $\|\tilde t - t^*\|_{L_2(\PP)}$. Therefore, we have
\$
\| \hat t - t^*\|_{L_2(\PP)} \lesssim  \min_{\tilde t\in\Phi_M(L,k,s)}  \|\tilde t-t^*\|_{L_2(\PP)} +    \gamma n^{-1/2}\log n  +  n^{-1/2} [\sqrt{\log(1/\varepsilon)} + \gamma^{-1}\log(1/\varepsilon) ]
\$
with probability at least $1-\varepsilon\cdot\exp(-\gamma^2)$. This concludes the proof of the theorem.

\subsection{Proof of Theorem \ref{theorem:norm_control_theorem}}\label{subsection:proof_norm_control_theorem}

We follow the proof in \cite{li2018tighter}. 
We denote by the loss function in \eqref{eq:problem_norm_control} as $\cL[t(x)] = f^\dagger(t(x^{\text{I}})) - t(x^{\text{II}})$, where $x^{\text{I}}$ follows the distribution $\PP$ and $x^{\text{II}}$ follows $\QQ$. To prove the theorem, we first link the generalization error in our theorem to the empirical Rademacher complexity (ERC). Given the data $\{x_i\}_{i = 1}^n$, the ERC related with the class $\cL(\Phi_{\text{norm}})$ is defined as
\#\label{eq:erc_def}
\mathfrak{R}_n [\cL(\Phi_{\text{norm}}) ] = \EE_\varepsilon \biggl[  \sup_{\varphi\in\Phi_{\text{norm}}}      \Bigl|  \frac{1}{n}\sum_{i = 1}^n \varepsilon_i \cdot \cL [\varphi(x_i; W, v) ]    \Bigr| \bigggiven \{x_i\}_{i = 1}^n  \biggr],
\#
where $\varepsilon_i$'s are i.i.d. Rademacher random variables, i.e., $\PP(\varepsilon_i = 1) = \PP(\varepsilon_i = -1) = 1/2$.  Here the expectation $\EE_\varepsilon(\cdot)$ is taken over the Rademacher random variables $\{\varepsilon_i\}_{i\in[n]}$.

We introduce the following lemma, which links the ERC to the generalization error bound.  

\begin{lemma}[\citep{mohri2018foundations}]\label{lemma:rdmk_thm}
Assume that $\sup_{\varphi\in\Phi_{\text{norm}}}|\cL(\varphi)|\leq M_1$, then for any $\varepsilon > 0$, with probability at least $1-\varepsilon$, we have
\$
\sup_{\varphi\in\Phi_{\text{norm}}}  \biggl\{\EE_x \{\cL [\varphi(x; W, v) ] \} - \frac{1}{n}\sum_{i = 1}^n   \cL [\varphi(x_i; W, v) ] \biggr\} \lesssim  \mathfrak{R}_n [\cL(\Phi_{\text{norm}}) ] + M_1 \cdot  n^{-1/2}\sqrt{\log(1/\varepsilon)},
\$
\end{lemma}
where the expectation $\EE_x\{\cdot\}$ is taken over $x^{\text{I}}\sim \PP$ and $x^{\text{II}}\sim \QQ$.

Equipped with Lemma \ref{lemma:rdmk_thm}, we only need to bound the ERC defined in \eqref{eq:erc_def}.

\begin{lemma}\label{lemma:tuo_zhao_thm1}
Let $\cL$ be a Lipschitz continuous loss function and $\Phi_{\text{norm}}$ be the family of networks defined in \eqref{eq:nn_norm}. We assume that the input $x\in\RR^d$ is bounded such that $\|x\|_2\leq B$. Then it holds that
\$
\mathfrak{R}_n [\cL(\Phi_{\text{norm}}) ] \lesssim \gamma_1\cdot n^{-1/2}\log(\gamma_2 n),
\$
where $\gamma_1$ and $\gamma_2$ are given in \eqref{eq:gamma_def}. 
\end{lemma}
We defer the proof to Section \S \ref{section:proof_thm_tuozhao}.

Now we proceed to prove the theorem. Recall that we assume that  $t^*\in\Phi_{\text{norm}}$. For notational convenience, we denote by 
\$
& \hat H(t) = \EE_{x\sim\PP_n} [ f^\dagger (t(x) )] - \EE_{x\sim\QQ_n} [t(x) ],\qquad H(t) = \EE_{x\sim\PP} [ f^\dagger (t(x) )  ] - \EE_{x\sim\QQ} [t(x) ].
\$
Then $\EE[\hat H(t)] = H(t)$. We proceed to bound $|\hat D_f(q\|p) - D_f(q\|p)| = |\hat H(\hat t) - H(t^*)|$. Note that if $\hat H(\hat t) \geq H(t^*)$, then we have
\#\label{eq:morton1}
0 \leq \hat H(\hat t) - H(t^*)\leq \hat H(t^*) - H(t^*),
\#
where the second inequality follows from the fact that $\hat t$ is the minimizer of $\hat H(\cdot)$. On the other hand, if $\hat H(\hat t) \leq H(t^*)$,  we have
\#\label{eq:morton2}
0 \geq  \hat H(\hat t) - H(t^*) \geq \hat H(\hat t) - H(\hat t),
\#
where the second inequality follows that fact that $t^*$ is the minimizer of $H(\cdot)$. Therefore, by \eqref{eq:morton1}, \eqref{eq:morton2}, and the fact that $\cL(\varphi)\lesssim \prod_{j = 1}^{L+1}B_j$ for any $\varphi\in\Phi_{\text{norm}}$, we deduce that
\#\label{eq:bound_norm1}
 |\hat H(\hat t) - H(t^*) |\leq \sup_{t\in\Phi_{\text{norm}}}  |\hat H(t) - H(t) | \lesssim \mathfrak{R}_n [\cL(\Phi_{\text{norm}}) ] + \prod_{j = 1}^{L+1}B_j\cdot n^{-1/2}\sqrt{\log(1/\varepsilon)}
\#
with probability at least $1-\varepsilon$. Here the second inequality follows from Lemma \ref{lemma:rdmk_thm}.  By plugging the result from Lemma \ref{lemma:tuo_zhao_thm1} into \eqref{eq:bound_norm1}, we deduce that with probability at least $1-\varepsilon$, it holds that  
\$
&|\hat D_f(q\|p) - D_f(q\|p) | = |\hat H(\hat t) - H(t^*) |  \lesssim \gamma_1\cdot n^{-1/2}\log(\gamma_2 n)+ \prod_{j = 1}^{L+1}B_j\cdot n^{-1/2}\sqrt{\log(1/\varepsilon)}. 
\$
This concludes the proof of the theorem.

\subsection{Proof of Theorem \ref{theorem:outer_prob_norm}}\label{subsection:proof_outer_prob_norm}
We first need to bound the max deviation of the estimated $f$-divergence $\hat D_f(q\|p)$ among all $q\in\cQ$. We utilize the following lemma to provide such a bound.  

\begin{lemma}\label{lemma:for_all_hold_norm}
Assume that the distribution $q$ is in the set $\cQ$, and we denote its $L_2$ covering number as $N_2(\delta, \cQ)$. Then for any target distribution $p$, we have
\$
\max_{q\in \cQ} |D_f(q\|p) - \hat D_f(q\|p) | \lesssim b_2(n,\gamma_1, \gamma_2) + \prod_{j = 1}^{L+1}B_j\cdot n^{-1/2}\cdot \sqrt{ \log ( N_2 [b_2(n,\gamma_1, \gamma_2), \cQ ]/\varepsilon  )}
\$
with probability at least $1- \varepsilon$. Here $b_2(n,\gamma_1, \gamma_2) = \gamma_1 n^{-1/2}\log (\gamma_2 n)$ and $c$ is a positive absolute constant.
\end{lemma}

We defer the proof to Section \S \ref{subsection:proof_for_all_hold_norm}. 

Now we turn to the proof of the theorem. We denote by $\tilde q' = \argmin_{\tilde q\in\cQ} D_f(\tilde q\| p)$. Then with probability at least $1- \varepsilon$, we have
\$ 
D_f(\hat q\|p) & \leq  |D_f(\hat q\|p) - \hat D_f(\hat q\|p) | + \hat D_f(\hat q\|p)\notag\\
&\leq \max_{q\in \cQ} |D_f(q\|p) - \hat D_f(q\|p) | + \hat D_f(\tilde q' \|p)\notag\\
&\lesssim b_2(n,\gamma_1, \gamma_2) + \prod_{j = 1}^{L+1}B_j\cdot n^{-1/2}\cdot  \sqrt{\log ( N_2 [b_2(n,\gamma_1, \gamma_2), \cQ ]/\varepsilon  )} + D_f(\tilde q' \|p),
\$ 
where we use the optimality of $\hat q$ among all $\tilde q\in \cQ$ to the problem \eqref{eq:q_outer_sample} in the second inequality, and we uses Lemma \ref{lemma:for_all_hold_norm} and Theorem \ref{thm:main} in the last line. Moreover, note that $D_f(\tilde q' \|p) = \min_{\tilde q\in\cQ} D_f(\tilde q\| p)$, we obtain that 
\$
D_f(\hat q\|p) \lesssim b_2(n,\gamma_1, \gamma_2) + \prod_{j = 1}^{L+1}B_j\cdot n^{-1/2} \sqrt{ \log ( N_2 [b_2(n,\gamma_1, \gamma_2), \cQ ]/\varepsilon  )} +  \min_{\tilde q\in\cQ}D_f(\tilde q\|p). 
\$
This concludes the proof of the theorem.

\section{Lemmas and Proofs}

\subsection{Proof of Lemma \ref{lemma:ineq1}}\label{subsection:proof_ineq1}
For any real-valued function $\varrho$, we write $\EE_{\PP}(\varrho) = \EE_{x\sim \PP}[\varrho(x)]$, $\EE_{\QQ}(\varrho) = \EE_{x\sim \QQ}[\varrho(x)]$, $\EE_{\PP_n}(\varrho) = \EE_{x\sim \PP_n}[\varrho(x)]$, and $\EE_{\QQ_n}(\varrho) = \EE_{x\sim \QQ_n}[\varrho(x)]$ for notational convenience. 

By the definition of $t^\natural$ in \eqref{eq:gen-form}, we have
\$
\EE_{\PP_n} [f^\dagger(t^\natural)  ] -\EE_{\QQ_n} (t^\natural) \leq \EE_{\PP_n} [ f^\dagger(t^*) ] - \EE_{\QQ_n}(t^*).
\$
Note that the functional $G(t) = \EE_{\PP_n} [f^\dagger(t)] - \EE_{\QQ_n}(t)$ is convex in $t$ since $f^\dagger$ is convex, we then have
\$
G (\frac{t^\natural + t^*}{2} ) - G(t^*)\leq \frac{G(t^\natural) - G(t^*)}{2}\leq 0.
\$
By re-arranging terms, we have
\#\label{eq:basic11}
& \{\EE_{\PP_n} [f^\dagger ((t^\natural + t^*)/2 ) - f^\dagger(t^*)  ] - \EE_{\PP} [f^\dagger ((t^\natural + t^*)/2 ) - f^\dagger(t^*)  ] \}  -  \{\EE_{\QQ_n} [(t^\natural - t^*)/2 ] - \EE_{\QQ} [(t^\natural - t^*)/2 ]  \} \notag\\
&\qquad \leq    \EE_{\QQ} [(t^\natural - t^*)/2 ] -   \EE_{\PP} [f^\dagger ((t^\natural + t^*)/2 ) - f^\dagger(t^*)  ]. 
\#
We denote by
\#\label{eq:df11}
B_f(t^*, t) =   \EE_\PP [f^\dagger(t) - f^\dagger(t^*)  ]  - \EE_\QQ (t-t^*). 
\#
then the RHS of \eqref{eq:basic11} is exactly $-B_f(t^*, (t^\natural + t^*)/2)$. We proceed to establish the lower bound of $B_f(t^*, t)$ using $L_2(\PP)$ norm. From $t^*(x; p, q) = f'(q(x)/p(x))$ and $(f^\dagger)'\circ (f')(x) = x$, we know that $q/p= {\partial f^\dagger}(t^*)/{\partial t}$. Then by substituting the second term on the RHS of \eqref{eq:df11} using the above relationship, we have
\$
B_f(t^*, t)& = \EE_\PP  \biggl[  f^\dagger(t) - f^\dagger(t^*) - \frac{\partial f^\dagger}{\partial t}(t^*)\cdot (t-t^*)   \biggr]
\$
Note that by Assumption \ref{assumption:reg} and Lemma \ref{lemma:strong_convex}, we know that the Fenchel duality $f^\dagger$ is strongly convex with parameter $1/L_0$. This gives that
\$
f^\dagger (t(x)) - f^\dagger (t^*(x) ) - \frac{\partial f^\dagger}{\partial t} (t^*(x) )\cdot  [t(x)-t^*(x) ] \geq  {1}/{L_0}\cdot   (t(x) - t^*(x) )^ 2
\$
for any $x$. Consequently, it holds that
\#\label{eq:bound-on-a12}
B_f(t^*, t) \geq 1/L_0\cdot \|t-t^*\|_{L_2(\PP)}^2. 
\#
By \eqref{eq:bound-on-a12}, we conclude that
\$
{1}/(4L_0)\cdot \|t^\natural-t^*\|_{L_2(\PP)}^2  & \leq  \{\EE_{\QQ_n} [(\hat t - t^*)/2 ] - \EE_{\QQ} [(\hat t - t^*)/2 ]  \} \\
&\qquad-  \{\EE_{\PP_n} [f^\dagger ((\hat t + t^*)/2 ) - f^\dagger(t^*)  ] - \EE_{\PP} [f^\dagger ((\hat t + t^*)/2 ) - f^\dagger(t^*)  ] \}. 
\$
This concludes the proof of the lemma.

\subsection{Proof of Lemma \ref{lemma:fake_vdg2}}\label{subsection:proof_fake_vdg2}
For any real-valued function $\varrho$, we write $\EE_{\PP}(\varrho) = \EE_{x\sim \PP}[\varrho(x)]$, $\EE_{\QQ}(\varrho) = \EE_{x\sim \QQ}[\varrho(x)]$, $\EE_{\PP_n}(\varrho) = \EE_{x\sim \PP_n}[\varrho(x)]$, and $\EE_{\QQ_n}(\varrho) = \EE_{x\sim \QQ_n}[\varrho(x)]$ for notational convenience. 

We first introduce the following concepts.  
For any $K > 0$, the Bernstein difference $\rho_{K, \PP}^2(t)$ of $t(\cdot)$ with respect to the distribution $\PP$ is defined to be
\$
\rho_{K, \PP}^2(t) = 2K^2\cdot \EE_{\PP}  [ \exp (|t|/K ) - 1 - |t|/K  ].
\$
Correspondingly, we denote by $\cH_{K, B}$ the generalized entropy with bracketing induced by the Bernstein difference $\rho_{K,\PP}$. We denote by $H_{s,B}$ the entropy with bracketing induced by $L_s$ norm, $H_{s}$ the entropy induced by $L_s$ norm, $H_{L_s(\PP),B}$ the entropy with bracketing induced by $L_s(\PP)$ norm, and $H_{L_s(\PP)}$ the regular entropy induced by $L_s(\PP)$ norm.

We consider the space
\$
\Psi = \psi(\Phi) =  \{\psi(t): t(x)\in \Phi \}. 
\$
For any $\delta > 0$, we denote the following space
\$
&\Psi(\delta) =  \{\psi(t)\in\Psi:  \|\psi(t)-\psi(t^*) \|_{L_2(\PP)}\leq \delta \},\\
&\Psi'(\delta) =  \{\Delta \psi(t) = \psi(t)-\psi(t^*): \psi(t)\in \Psi(\delta) \}.
\$
Note that $\sup_{\Delta \psi(t)\in\Psi'(\delta)}\| \Delta \psi(t)\|_\infty\leq 2M_0$ and $\sup_{ \Delta \psi(t)\in\Psi'(\delta)}\|\Delta \psi(t)\|_\infty\leq \delta$, by Lemma \ref{lemma:bound_rho} we have
\$ 
\sup_{\Delta \psi(t)\in \Psi'(\delta)} \rho_{8M_0, \PP} [\Delta \psi(t) ]\leq \sqrt{2}\delta.
\$
To invoke Theorem \ref{thm:vandegeer} for $\cG = \Psi'(\delta)$, we pick $K = 8M_0$. By the fact that $\sup_{\Delta \psi(t)\in\Psi'(\delta)}\| \Delta \psi(t)\|_\infty\leq 2M_0$, Lemma \ref{lemma:cover_number}, Assumption \ref{assum:phi-cn}, and the fact that $\psi$ is Lipschitz continuous, we have
\$
\cH_{8M_0, B} (u, \Psi'(\delta), \PP )\leq H_{2, B} (\sqrt{2} u, \Psi'(\delta), \PP )\leq u^{-\gamma_\Phi}
\$
for any $u>0$. Then, by algebra, we have the follows
\$
\int_0^R\cH_{8M_0, B}^{1/2} (u, \Psi'(\delta), \PP  )\ud u\leq  \frac{2}{2-\gamma_\Phi} R^{-\gamma_\Phi / 2 + 1} .
\$
We take $C = 1$, and $a, C_1$ and $C_0$ in Theorem \ref{thm:vandegeer} to be
\$
&a = C_1\sqrt{n} R^2 / K,\qquad C_0 = 2C^2C_2 \vee 2C,\qquad C_1 = C_0 C_2, 
\$
where $C_2$ is a sufficiently large constant. 
Then it is straightforward to check that our choice above satisfies the conditions in Theorem \ref{thm:vandegeer} for any $\delta$ such that $\delta \geq n^{-1/(\gamma_\Phi + 2)}$, when $n$ is sufficiently large. With $\delta_n = n^{-1/(\gamma_\Phi + 2)}$, we have
\$
& \PP \biggl\{\sup_{t\colon \psi(t)\in \Psi, \psi(t)\notin \Psi(\delta_n)}  \frac{ |\EE_{\PP_n} [\psi(t) - \psi(t^*) ]   - \EE_{\PP} [\psi(t) -  \psi(t^*) ]     |}{  n^{-2/(\gamma_\Phi + 2)} } \geq C_1/K  \biggr\}\notag\\
&\qquad \leq  \PP \biggl\{\sup_{t\colon \psi(t)\in \Psi, \psi(t)\notin \Psi(\delta_n)}  \frac{ |\EE_{\PP_n} [\psi(t) - \psi(t^*) ]   - \EE_{\PP} [\psi(t) -  \psi(t^*) ]     |}{ \|\psi(t)-\psi(t^*) \|_{L_2(\PP)}^2 } \geq C_1/K  \biggr\}\notag\\
& \qquad \leq \sum_{s = 0}^S  \PP \biggl\{\sup_{t\colon \psi(t)\in \Psi, \psi(t)\in \Psi(2^{s+1}\delta_n)}  { |\EE_{\PP_n} [\psi(t) - \psi(t^*) ]   - \EE_{\PP} [\psi(t) -  \psi(t^*) ]     |}  \geq C_1/K\cdot  (2^{s} \delta_n)^2  \biggr\}\notag\\
& \qquad \leq \sum_{s = 0}^S C \exp\biggl(  -\frac{C_1^2/K^2\cdot 2^{2s}\cdot n^{\gamma_\Phi / (2 + \gamma_\Phi)}}{C^2(C_1 + 1)}\biggr) \leq  c_1 \exp(-n^{\gamma_\Phi / (2 + \gamma_\Phi)} / c_1^2),
\$
for some constant $c_1 > 0$.  Here in the last line, we invoke Theorem \ref{thm:vandegeer} with $R = 2^s \delta_n$.   Therefore, we have 
\$
& \PP \biggl\{\sup_{t\colon \psi(t)\in \Psi}  \frac{ |\EE_{\PP_n} [\psi(t) - \psi(t^*) ]   - \EE_{\PP} [\psi(t) -  \psi(t^*) ]     |}{  n^{-2/(\gamma_\Phi + 2)} } \geq C_1/K  \biggr\} \leq c_1 \exp(-n^{\gamma_\Phi / (2 + \gamma_\Phi)} / c_1^2). 
\$
We conclude the proof of Lemma \ref{lemma:fake_vdg2}.

\subsection{Proof of Lemma \ref{lemma:for_all_hold_sparsity}}\label{subsection:proof_for_all_hold_sparsity}
Recall that the covering number of $\cQ$ is $N_2(\delta, \cQ)$, we thus assume that there exists $q_{1}, \ldots, q_{N_2(\delta, \cQ)}\in\cQ$ such that for any $q\in \cQ$, there exists some $q_{k}$, where $1\leq k\leq N_2(\delta, \cQ)$, so that $\|q-q_{k}\|_2\leq \delta$. Moreover, by taking $\delta = \delta_n =  n^{-{1}/{(\gamma_\Phi+2)}}$ and union bound, we have
\$
& \PP [  \sup_{q\in\cQ}  | D_f(q\|p) -  D_f^\natural(q\|p)   |  \geq c_1\cdot   n^{-{1}/{(\gamma_\Phi+2)}} \cdot \log n  ] \\
& \qquad \leq \sum_{k = 1}^{N_2(\delta_n, \cQ)} \PP [   | D_f(q_k\|p) -  D_f^\natural(q_k\|p)   |  \geq c_1\cdot   n^{-{1}/{(\gamma_\Phi+2)}} \cdot \log n  ]\\
&\qquad \leq N_2(\delta_n, \cQ)\cdot \exp(-n^{{\gamma_\Phi}/{(\gamma_\Phi + 2)}}\cdot \log n),
\$
where the last line comes from Theorem \ref{thm:gen}.  Combining Assumption \ref{assum:cov_num1}, when $n$ is sufficiently large, it holds that 
\$
\PP [  \sup_{q\in\cQ}  | D_f(q\|p) - D_f^\natural(q\|p)   |  \geq c_1\cdot  n^{-{1}/{(\gamma_\Phi+2)}} \cdot \log n  ]  \leq 1/n,
\$
which concludes the proof of the lemma.

\subsection{Proof of Lemma \ref{lemma:for_all_hold_sparsity2}}\label{subsection:proof_for_all_hold_sparsity2}
Recall that the covering number of $\cQ$ is $N_2(\delta, \cQ)$, we thus assume that there exists $q_{1}, \ldots, q_{N_2(\delta, \cQ)}\in\cQ$ such that for any $q\in \cQ$, there exists some $q_{k}$, where $1\leq k\leq N_2(\delta, \cQ)$, so that $\|q-q_{k}\|_2\leq \delta$. Moreover, by taking $\delta = \delta_n =  n^{-{\beta}/{(d+2\beta)}}$ and union bound, we have
\$
& \PP [  \sup_{q\in\cQ}  | D_f(q\|p) - \hat D_f(q\|p)   |  \geq c_1\cdot   n^{-{\beta}/{(d+2\beta)}} \cdot \log^7 n  ] \\
& \qquad \leq \sum_{k = 1}^{N_2(\delta_n, \cQ)} \PP [   | D_f(q_k\|p) - \hat D_f(q_k\|p)   |  \geq c_1\cdot   n^{-{\beta}/{(d+2\beta)}} \cdot \log^7 n  ]\\
&\qquad \leq N_2(\delta_n, \cQ)\cdot \exp(-n^{-{d}/{(d+2\beta)}}\cdot \log n),
\$
where the last line comes from Theorem \ref{thm:main}.  Combining Assumption \ref{assum:cov_num}, when $n$ is sufficiently large, it holds that 
\$
\PP [  \sup_{q\in\cQ}  | D_f(q\|p) - \hat D_f(q\|p)   |  \geq c_1\cdot  n^{-{\beta}/{(d+2\beta)}} \cdot \log^7 n  ]  \leq 1/n,
\$
which concludes the proof of the lemma.

\subsection{Proof of Lemma \ref{lemma:ineq2}}\label{subsection:proof_ineq2}
For any real-valued function $\varrho$, we write $\EE_{\PP}(\varrho) = \EE_{x\sim \PP}[\varrho(x)]$, $\EE_{\QQ}(\varrho) = \EE_{x\sim \QQ}[\varrho(x)]$, $\EE_{\PP_n}(\varrho) = \EE_{x\sim \PP_n}[\varrho(x)]$, and $\EE_{\QQ_n}(\varrho) = \EE_{x\sim \QQ_n}[\varrho(x)]$ for notational convenience. 

By the definition of $\hat t$ in \eqref{problem_formulation}, we have
\$
\EE_{\PP_n} [f^\dagger(\hat t)  ] -\EE_{\QQ_n} (\hat t) \leq \EE_{\PP_n} [ f^\dagger(\tilde t) ] - \EE_{\QQ_n}(\tilde t).
\$
Note that the functional $G(t) = \EE_{\PP_n} [f^\dagger(t)] - \EE_{\QQ_n}(t)$ is convex in $t$ since $f^\dagger$ is convex, we then have
\$
G (\frac{\hat t + \tilde t}{2} ) - G(\tilde t)\leq \frac{G(\hat t) - G(\tilde t)}{2}\leq 0.
\$
By re-arranging terms, we have
\#\label{eq:basic1}
& \{\EE_{\PP_n} [f^\dagger ((\hat t + \tilde t)/2 ) - f^\dagger(\tilde t)  ] - \EE_{\PP} [f^\dagger ((\hat t + \tilde t)/2 ) - f^\dagger(\tilde t)  ] \}  -  \{\EE_{\QQ_n} [(\hat t - \tilde t)/2 ] - \EE_{\QQ} [(\hat t - \tilde t)/2 ]  \} \notag\\
&\qquad \leq    \EE_{\QQ} [(\hat t - \tilde t)/2 ] -   \EE_{\PP} [f^\dagger ((\hat t + \tilde t)/2 ) - f^\dagger(\tilde t)  ]. 
\#
We denote by
\#\label{eq:df1}
B_f(\tilde t, t) =   \EE_\PP [f^\dagger(t) - f^\dagger(\tilde t)  ]  - \EE_\QQ (t-\tilde t). 
\#
then the RHS of \eqref{eq:basic1} is exactly $-B_f(\tilde t, (\hat t + \tilde t)/2)$. We proceed to establish the lower bound of $B_f(\tilde t, t)$ using $L_2(\PP)$ norm. From $t^*(x; p, q) = f'(q(x)/p(x))$ and $(f^\dagger)'\circ (f')(x) = x$, we know that $q/p= {\partial f^\dagger}(t^*)/{\partial t}$. Then by substituting the second term on the RHS of \eqref{eq:df1} using the above relationship, we have
\#\label{eq:df2}
B_f(\tilde t, t)& = \EE_\PP  \biggl[  f^\dagger(t) - f^\dagger(\tilde t) - \frac{\partial f^\dagger}{\partial t}(t^*)\cdot (t-\tilde t)   \biggr]\notag\\
& = \EE_\PP  \biggl[  f^\dagger(t) - f^\dagger(\tilde t) - \frac{\partial f^\dagger}{\partial t}(\tilde t)\cdot (t-\tilde t)   \biggr]  + \EE_\PP \biggl\{  \biggl[\frac{\partial f^\dagger}{\partial t}(\tilde t) - \frac{\partial f^\dagger}{\partial t}( t^*)  \biggr]\cdot (t-\tilde t) \biggr\} \notag\\
& = A_1 + A_2.
\#
We lower bound $A_1$ and $A_2$ in the sequel. 

\vskip5pt
\noindent \textbf{Bound on $A_1$. }
Note that by Assumption \ref{assumption:reg} and Lemma \ref{lemma:strong_convex}, we know that the Fenchel duality $f^\dagger$ is strongly convex with parameter $1/L_0$. This gives that
\$
f^\dagger (t(x)) - f^\dagger (\tilde t(x) ) - \frac{\partial f^\dagger}{\partial t} (\tilde t(x) )\cdot  [t(x)-\tilde t(x) ] \geq  {1}/{L_0}\cdot   (t(x) - \tilde t(x) )^ 2
\$
for any $x$. Consequently, it holds that
\#\label{eq:bound-on-a1}
A_1\geq 1/L_0\cdot \|t-\tilde t\|_{L_2(\PP)}^2. 
\#

\vskip5pt
\noindent \textbf{Bound on $A_2$. }
By Cauchy-Schwarz inequality, it holds that
\$
A_2\geq -\sqrt{\EE_\PP \biggl\{  \biggl[\frac{\partial f^\dagger}{\partial t}(\tilde t) - \frac{\partial f^\dagger}{\partial t}( t^*)\biggr  ]^2 \biggr\}} \cdot \sqrt{\EE_\PP [ (t -\tilde t)^2 ]}.
\$
Again, by Assumption \ref{assumption:reg} and Lemma \ref{lemma:strong_convex}, we know that the Fenchel duality $f^\dagger$ has $1/\mu_0$-Lipschitz gradient, which gives that
\$
\biggl |\frac{\partial f^\dagger}{\partial t} (\tilde t(x) ) - \frac{\partial f^\dagger}{\partial t} ( t^*(x) )  \biggr| \leq 1/\mu_0 \cdot  |\tilde t(x)-t^*(x) |
\$
for any $x$. By this, the term $A_2$ is lower bounded:
\#\label{eq:bound-on-a2}
A_2\geq -1/\mu_0\cdot \|\tilde t - t^*\|_{L_2(\PP)}\cdot \|t - \tilde t\|_{L_2(\PP)}. 
\#

Plugging \eqref{eq:bound-on-a1} and \eqref{eq:bound-on-a2} into \eqref{eq:df2}, we have
\$
B_f(\tilde t, t) \geq 1/L_0\cdot \|t-\tilde t\|_{L_2(\PP)}^2 -1/\mu_0\cdot \|\tilde t - t^*\|_{L_2(\PP)}\cdot \|t - \tilde t\|_{L_2(\PP)}. 
\$
By this, together with \eqref{eq:basic1}, we conclude that
\$
{1}/(4L_0)\cdot \|\hat t-\tilde t\|_{L_2(\PP)}^2  & \leq  {1}/{\mu_0}\cdot  \|\hat t - \tilde t\|_{L_2(\PP)}\cdot \|\tilde t - t^*\|_{L_2(\PP)} +   \{\EE_{\QQ_n} [(\hat t - \tilde t)/2 ] - \EE_{\QQ} [(\hat t - \tilde t)/2 ]  \} \\
&\qquad-  \{\EE_{\PP_n} [f^\dagger ((\hat t + \tilde t)/2 ) - f^\dagger(\tilde t)  ] - \EE_{\PP} [f^\dagger ((\hat t + \tilde t)/2 ) - f^\dagger(\tilde t)  ] \}. 
\$
This concludes the proof of the lemma.

\subsection{Proof of Lemma \ref{lemma:fake_vdg}}\label{subsection:proof_fake_vdg}
For any real-valued function $\varrho$, we write $\EE_{\PP}(\varrho) = \EE_{x\sim \PP}[\varrho(x)]$, $\EE_{\QQ}(\varrho) = \EE_{x\sim \QQ}[\varrho(x)]$, $\EE_{\PP_n}(\varrho) = \EE_{x\sim \PP_n}[\varrho(x)]$, and $\EE_{\QQ_n}(\varrho) = \EE_{x\sim \QQ_n}[\varrho(x)]$ for notational convenience. 

We first introduce the following concepts.  
For any $K > 0$, the Bernstein difference $\rho_{K, \PP}^2(t)$ of $t(\cdot)$ with respect to the distribution $\PP$ is defined to be
\$
\rho_{K, \PP}^2(t) = 2K^2\cdot \EE_{\PP}  [ \exp (|t|/K ) - 1 - |t|/K  ].
\$
Correspondingly, we denote by $\cH_{K, B}$ the generalized entropy with bracketing induced by the Bernstein difference $\rho_{K,\PP}$. We denote by $H_{s,B}$ the entropy with bracketing induced by $L_s$ norm, $H_{s}$ the entropy induced by $L_s$ norm, $H_{L_s(\PP),B}$ the entropy with bracketing induced by $L_s(\PP)$ norm, and $H_{L_s(\PP)}$ the regular entropy induced by $L_s(\PP)$ norm.

Since we focus on fixed $L$, $k$, and $s$, we denote by $\Phi_M = \Phi_M(L,k,s)$ for notational convenience.  
We consider the space
\$
\Psi_M = \psi(\Phi_M) =  \{\psi(t): t(x)\in \Phi_M \}. 
\$
For any $\delta > 0$, we denote the following space
\$
&\Psi_M(\delta) =  \{\psi(t)\in\Psi_M:  \|\psi(t)-\psi(\tilde t) \|_{L_2(\PP)}\leq \delta \},\\
&\Psi'_M(\delta) =  \{\Delta \psi(t) = \psi(t)-\psi(\tilde t): \psi(t)\in \Psi_M(\delta) \}.
\$
Note that $\sup_{\Delta \psi(t)\in\Psi'_M(\delta)}\| \Delta \psi(t)\|_\infty\leq 2M_0$ and $\sup_{ \Delta \psi(t)\in\Psi'_M(\delta)}\|\Delta \psi(t)\|_\infty\leq \delta$, by Lemma \ref{lemma:bound_rho} we have
\$ 
\sup_{\Delta \psi(t)\in \Psi'_M(\delta)} \rho_{8M_0, \PP} [\Delta \psi(t) ]\leq \sqrt{2}\delta.
\$
To invoke Theorem \ref{thm:vandegeer} for $\cG = \Psi'_M(\delta)$, we pick $K = 8M_0$ and $R = \sqrt{2}\delta$. Note that from the fact that $\sup_{\Delta \psi(t)\in\Psi'_M(\delta)}\| \Delta \psi(t)\|_\infty\leq 2M_0$, by Lemma \ref{lemma:cover_number}, Lemma \ref{lemma:cover_number_nn}, and the fact that $\psi$ is Lipschitz continuous, we have
\$
\cH_{8M_0, B} (u, \Psi'_M(\delta), \PP )\leq H_{\infty} (u/(2\sqrt{2}), \Psi'_M(\delta) )\leq 2(s+1)\log (4\sqrt{2}u^{-1}(L+1)V^2 )
\$
for any $u>0$. Then, by algebra, we have the follows
\$
\int_0^R\cH_{8M_0, B}^{1/2} (u, \Psi'_M(\delta), \PP  )\ud u\leq 3  s^{1/2} \delta \cdot \log({8 V^2L}/{\delta}).
\$
For any $0<\varepsilon<1$, we take $C = 1$, and $a, C_1$ and $C_0$ in Theorem \ref{thm:vandegeer} to be
\$
&a = 8M_0\log (\exp(\gamma^2)/\varepsilon )\gamma^{-1}\cdot \delta,\notag\\
&C_0 = 6M_0\gamma^{-1}\sqrt{\log (\exp(\gamma^2)/\varepsilon )},\notag\\
&C_1 = 33 M_0^2 \gamma^{-2}\log (\exp(\gamma^2)/\varepsilon ).
\$
Here $\gamma = s^{1/2}\log(V^2L)$. Then it is straightforward to check that our choice above satisfies the conditions in Theorem \ref{thm:vandegeer} for any $\delta$ such that $\delta \geq \gamma n^{-1/2}$, when $n$ is sufficiently large such that $n\gtrsim [\gamma + \gamma^{-1}\log(1/\varepsilon)]^2$. Consequently, by Theorem \ref{thm:vandegeer}, for $\delta \geq \gamma n^{-1/2}$, we have
\$ 
&\PP \{\sup_{t(x)\in\Phi_M(\delta)} |\EE_{\PP_n} [\psi(t) - \psi(\tilde t) ]   - \EE_{\PP} [\psi(t) -  \psi(\tilde t) ]  |\geq 8M_0 \log(\exp(\gamma^2)/\varepsilon)\gamma^{-1} \cdot \delta\cdot n^{-1/2}  \}\notag\\
&\qquad = \PP \{\sup_{\Delta \psi(t)\in\Psi_M'(\delta)} |\EE_{\PP_n} [\Delta\psi(t) ]   - \EE_{\PP} [\Delta\psi(t) ] |\geq 8M_0 \log(\exp(\gamma^2)/\varepsilon) \gamma^{-1}\cdot \delta\cdot n^{-1/2}  \} \notag\\
&\qquad\leq \varepsilon\cdot\exp(-\gamma^2).
\$
By taking $\delta = \delta_n =  \gamma n^{-1/2}$, we have
\#\label{eq:ret_p1}
\PP \biggl\{\sup_{t(x)\in\Phi_M(\delta)}\frac{ |\EE_{\PP_n} [\psi(t) - \psi(\tilde t) ]   - \EE_{\PP} [\psi(t) -  \psi(\tilde t) ]  |}{n^{-1} [ \gamma^2 + \log(1/\varepsilon)]}\leq 8M_0   \biggr\}\geq 1-\varepsilon\cdot\exp(-\gamma^2). 
\#
On the other hand, we denote that $S = \min\{s>1: 2^{-s}(2M_0) < \delta_n\} = \cO(\log(\gamma^{-1} n^{1/2}))$. For notational convenience, we denote the set 
\#\label{eq:A_s}
A_s =  \{\psi(t)\in\Psi_M: \psi(t)\in \Psi_M(2^{-s+2}M_0), \psi(t)\notin \Psi_M(2^{-s+1}M_0) \}. 
\#  
Then by the peeling device, we have the following
\$
&\PP \biggl\{\sup_{\psi(t)\in \Psi_M, \psi(t)\notin \Psi_M(\delta_n)}  \frac{ |\EE_{\PP_n} [\psi(t) - \psi(\tilde t) ]   - \EE_{\PP} [\psi(t) -  \psi(\tilde t) ]     |}{ \|\psi(t)-\psi(\tilde t) \|_{L_2(\PP)}\cdot T(n,\gamma,\varepsilon)} \geq 16M_0  \biggr\}\\
& \qquad \leq \sum_{s = 1}^S\PP \biggl\{\sup_{\psi(t)\in A_s}  \frac{ |\EE_{\PP_n} [\psi(t) - \psi(\tilde t) ]   - \EE_{\PP} [\psi(t) -  \psi(\tilde t) ]  |}{ 2^{-s+1}M_0}  \geq 16M_0 \cdot  T(n,\gamma, \varepsilon) \biggr \}\\
& \qquad \leq \sum_{s = 1}^S\PP \{\sup_{\psi(t)\in A_s}   |\EE_{\PP_n} [\psi(t) - \psi(\tilde t) ]   - \EE_{\PP} [\psi(t) -  \psi(\tilde t) ]  | \geq 8 M_0   \cdot(2^{-s+2}M_0) \cdot T(n,\gamma, \varepsilon)  \}\\
& \qquad \leq \sum_{s = 1}^S\PP \{\sup_{\psi(t)\in \Psi_M(2^{-s+2}M_0)}   |\EE_{\PP_n} [\psi(t) - \psi(\tilde t) ]   - \EE_{\PP} [\psi(t) -  \psi(\tilde t) ]  |\geq 8 M_0    \cdot (2^{-s+2}M_0) \cdot T(n,\gamma, \varepsilon)  \}\\
& \qquad \leq S\cdot \varepsilon\cdot\exp(-\gamma^2)/\log(\gamma^{-1} n^{1/2}) = c\cdot\varepsilon\cdot\exp(-\gamma^2),
\$
where $c$ is a positive absolute constant, and for notational convenience we denote by $T(n,\gamma, \varepsilon) = \gamma^{-1}\cdot  n^{-1/2} \log (\log(\gamma^{-1} n^{1/2})\exp(\gamma^2)/\varepsilon)$. Here in the second line, we use the fact that for any $\psi(t)\in A_s$, we have $\|\psi(t) - \psi(\tilde t)\|_{L_2(\QQ)}\geq 2^{-s+1}M_0$ by the definition of $A_s$ in \eqref{eq:A_s}; in the forth line, we use the argument that since $A_s \subseteq \Psi_M(2^{-s+2}M_0)$, the probability of supremum taken over $\Psi_M(2^{-s+2}M_0)$ is larger than the one over $A_s$; in the last line we invoke Theorem \ref{thm:vandegeer}. Consequently, this gives us
\#\label{eq:ret_p2}
&\PP \biggl\{\sup_{\substack{\psi(t)\in \Psi_M\\ \psi(t)\notin \Psi_M(\delta_n)}}  \frac{ |\EE_{\PP_n} [\psi(t) - \psi(\tilde t) ]   - \EE_{\PP} [\psi(t) -  \psi(\tilde t) ]  |}{ \|\psi(t)-\psi(\tilde t) \|_{L_2(\PP)}\cdot n^{-1/2} [ \gamma\log n + \gamma^{-1}\log(1/\varepsilon)    ]} \leq 16M_0  \biggr\} \geq 1- \varepsilon\cdot\exp(-\gamma^2).
\#
Combining \eqref{eq:ret_p1} and \eqref{eq:ret_p2}, we finish the proof of the lemma.

\subsection{Proof of Lemma \ref{lemma:tuo_zhao_thm1}}\label{section:proof_thm_tuozhao}

The proof of the theorem utilizes following two lemmas. The first lemma characterizes the Lipschitz property of $\varphi(x; W, v)$ in the input $x$. 
\begin{lemma}\label{lemma:lip_x}
Given $W$ and $v$, then for any $\varphi(\cdot; W, v)\in \Phi_{\text{norm}}$ and $x_1, x_2\in \RR^{d}$, we have
\$
 \|\varphi(x_1; W, v) - \varphi(x_2; W, v) \|_2\leq \|x_1 - x_2\|_2\cdot \prod_{j = 1}^{L+1}B_j. 
\$
\end{lemma}
We defer the proof to Section \S \ref{subsection:proof_lip_x}.

The following lemma characterizes the Lipschitz property of $\varphi(x; W, v)$ in the network parameter pair $(W,v)$. 
\begin{lemma}\label{lemma:lip_w}
Given any bounded $x\in\RR^{d}$ such that $\|x\|_2\leq B$, then for any weights $W^1 = \{W_j^1\}_{j = 1}^{L+1}, W^2 = \{W_j^2\}_{j = 1}^{L+1}, v^1 = \{v^1_j\}_{j = 1}^L, v^2 = \{v^2_j\}_{j = 1}^L$, and functions $\varphi(\cdot, W^1, v^1), \varphi(\cdot, W^2, v^2)\in\Phi_{\text{norm}}$, we have
\$
&  \|\varphi(x, W^1, v^1) - \varphi(x, W^2, v^2) \|\notag\\
 &\qquad \leq \frac{B\sqrt{2L+1}\cdot  \prod_{j = 1}^{L+1}B_j}{\min_j B_j}\cdot \sum_{j = 1}^L A_j \cdot \sqrt{\sum_{j = 1}^{L+1}\|W_j^1 - W_j^2\|_\F^2 + \sum_{j = 1}^L \|v_j^1 - v_j^2\|_2^2}. 
\$
\end{lemma}

We defer the proof to Section \S \ref{subsection:proof_lip_w}.

We now turn to the proof of Lemma \ref{lemma:tuo_zhao_thm1}. Note that by Lemma \ref{lemma:lip_w}, we know that $\varphi(x; W, v)$ is $L_w$-Lipschitz in the parameter $(W, v)\in\RR^{b}$, where the dimension $b$ takes the form
\#\label{eq:def_dim_b}
b = \sum_{j = 1}^{L+1} k_j k_{j-1} + \sum_{j = 1}^L k_j\leq \sum_{j = 0}^{L+1}(k_j+1)^2,
\#
and the Lipschitz constant $L_w$ satisfies
\#\label{eq:def_lip_L_w}
L_w = \frac{B\sqrt{2L+1}\cdot \prod_{j = 1}^{L+1}B_j}{\min_j B_j}\cdot\sum_{j = 1}^L A_j .
\#
In addition, we know that the covering number of $\cW = \{(W,v)\in\RR^b: \sum_{j = 1}^{L+1} \|W_j\|_\F + \sum_{j =1}^L \|v_j\|_2 \leq K\}$, where 
\#\label{eq:def_K}
K = \sqrt{\sum_{j = 1}^{L+1}k_j^2B_j^2} + \sum_{j = 1}^L A_j,
\#
satisfies
\$
N(\cW, \delta) \leq  ( {3K}{\delta^{-1}} )^b.
\$ 
By the above facts, we deduce that the covering number of $\cL(\Phi_{\text{norm}})$ satisfies
\$
N [\cL(\Phi_{\text{norm}}), \delta ] \leq  ( {c_1K L_w}{\delta^{-1}} )^b,
\$
for some positive absolute constant $c_1$. Then by Dudley entropy integral bound on the ERC,  we know that
\#\label{eq:dud1}
\mathfrak{R}_n [\cL(\Phi_{\text{norm}}) ] \leq \inf_{\tau > 0} \tau + \frac{1}{\sqrt{n}}\int_\tau^{\vartheta}\sqrt{\log N [\cL(\Phi_{\text{norm}}), \delta ]} \ud\delta,
\#
where $\vartheta = \sup_{g(\cdot; W, v)\in \cL(\Phi_{\text{norm}}), x\in\RR^d} |g(x; W, v) |$. Moreover, from Lemma \ref{lemma:lip_x} and the fact that the loss function is Lipschitz continuous, we have
\#\label{eq:def_vartheta}
\vartheta \leq c_2\cdot B\cdot \prod_{j = 1}^{L+1}B_j
\#
for some positive absolute constant $c_2$. Therefore, by calculations, we derive from \eqref{eq:dud1} that
\$
\mathfrak{R}_n [\cL(\Phi_{\text{norm}}) ] = \cO \biggl( \frac{\vartheta}{\sqrt{n}}\cdot \sqrt{b\cdot \log\frac{KL_w \sqrt{n}}{\vartheta \sqrt{b}}}  \biggr),
\$
then we conclude the proof of the lemma by plugging in \eqref{eq:def_dim_b}, \eqref{eq:def_lip_L_w}, \eqref{eq:def_K}, and \eqref{eq:def_vartheta}, and using the definition of $\gamma_1$ and $\gamma_2$ in \eqref{eq:gamma_def}.

\subsection{Proof of Lemma \ref{lemma:for_all_hold_norm}}\label{subsection:proof_for_all_hold_norm}
Remember that the covering number of $\cQ$ is $N_2(\delta, \cQ)$, we assume that there exists $q_{1}, \ldots, q_{N_2(\delta, \cQ)}\in\cQ$ such that for any $q\in \cQ$, there exists some $q_{k}$, where $1\leq k\leq N_2(\delta, \cQ)$, so that $\|q-q_{k}\|_2\leq \delta$. Moreover, by taking $\delta =  \gamma_1 n^{-1/2}\log(\gamma_2 n) =  b_2(n,\gamma_1, \gamma_2)$ and $N_2 = N_2[b_2(n,\gamma_1, \gamma_2), \cQ]$, we have
\$
&\PP \{ \max_{q\in \cQ}  |D_f(q\|p) - \hat D_f(q\|p) |  \geq c\cdot  [ b_2(n,\gamma_1, \gamma_2) + \prod_{j = 1}^{L+1}B_j\cdot n^{-1/2}\cdot  \sqrt{ \log (N_2/\varepsilon)}  ]   \}\\
&\qquad \leq  \sum_{k = 1}^{N_2}\PP \{   |D_f(q\|p) - \hat D_f(q\|p) |  \geq c\cdot  [ b_2(n,\gamma_1, \gamma_2) + \prod_{j = 1}^{L+1}B_j\cdot n^{-1/2}\cdot    \sqrt{\log (N_2/\varepsilon)}  ]   \}\\
&\qquad \leq  N_2 \cdot \varepsilon/N_2 =\varepsilon,
\$
where the second line comes from union bound, and the last line comes from Theorem \ref{theorem:norm_control_theorem}. By this, we conclude the proof of the lemma.

\subsection{Proof of Lemma \ref{lemma:lip_x}}\label{subsection:proof_lip_x}
The proof follows by applying the Lipschitz property and bounded spectral norm of $W_j$ recursively:
\$
& \|\varphi(x_1; W, v) - \varphi(x_2; W, v) \|_2 =  \|W_{L+1}  (\sigma_{v_L}\cdots W_2\sigma_{v_1}W_1x_1 - \sigma_{v_L}\cdots W_2\sigma_{v_1}W_1x_2 ) \|_2\\
&\qquad\leq \|W_{L+1}\|_2\cdot  \| \sigma_{v_L} (W_L\cdots W_2\sigma_{v_1}W_1x_1 - W_L\cdots W_2\sigma_{v_1}W_1x_2  ) \|_2\\
&\qquad \leq B_{L+1}\cdot  \| W_L\cdots W_2\sigma_{v_1}W_1x_1 - W_L\cdots W_2\sigma_{v_1}W_1x_2  \|_2\\
&\qquad \leq\cdots\leq \prod_{j = 1}^{L+1}B_j\cdot \|x_1-x_2\|_2. 
\$
Here in the third line we uses the fact that $\|W_j\|_2\leq B_j$ and the $1$-Lipschitz property of $\sigma_{v_j}(\cdot)$, and in the last line we recursively apply the same argument as in the above lines. This concludes the proof of the lemma.

\subsection{Proof of Lemma \ref{lemma:lip_w}}\label{subsection:proof_lip_w}
Recall that $\varphi(x;W,v)$ takes the form 
\$
\varphi(x;W,v) = W_{L+1} \sigma_{v_L}W_L\cdots \sigma_{v_1}W_1x.
\$
For notational convenience, we denote by $\varphi_j^i(x) = \sigma_{v_j^i}(W_j^i x)$ for $i = 1,2$. By this, $\varphi(x;W,v)$ has the form $\varphi(x;W^i,v^i) = W_{L+1}^i\varphi^i_L\circ\cdots\circ\varphi^i_1(x)$. First, note that for any $W^1, W^2, v^1$ and $v^2$, by triangular inequality, we have
\#\label{eq:lip_w1}
&  \|\varphi(x,W^1, v^1) - \varphi(x,W^2, v^2) \|_2 =  \|W_{L+1}^1\varphi^1_L\circ\cdots\circ\varphi^1_1(x) - W_{L+1}^2\varphi^2_L\circ\cdots\circ\varphi^2_1(x) \|_2\notag\\
&\qquad \leq  \|W_{L+1}^1\varphi^1_L\circ\cdots\circ\varphi^1_1(x) - W_{L+1}^2\varphi^1_L\circ\cdots\circ\varphi^1_1(x) \|_2\notag\\
&\qquad\qquad\qquad\qquad +  \|W_{L+1}^2\varphi^1_L\circ\cdots\circ\varphi^1_1(x) - W_{L+1}^2\varphi^2_L\circ\cdots\circ\varphi^2_1(x) \|_2\notag\\
&\qquad \leq \|W_{L+1}^1 - W_{L+1}^2\|_\F \cdot  \|\varphi^1_L\circ\cdots\circ\varphi^1_1(x) \|_2 \notag\\
&\qquad\qquad\qquad\qquad +B_{L+1}\cdot   \|\varphi^1_L\circ\cdots\circ\varphi^1_1(x) - \varphi^2_L\circ\cdots\circ\varphi^2_1(x) \|_2.
\#
Moreover, note that for any $\ell\in[L]$, we have the following bound on $\|\varphi^1_L\circ\cdots\circ\varphi^1_1(x)\|_2$:
\#\label{eq:rev-1}
& \|\varphi^i_\ell\circ\cdots\circ\varphi^i_1(x) \|_2 \leq  \|W_\ell^i\varphi^i_{\ell-1}\circ\cdots\circ\varphi^i_1(x) \|_2 + \|v_\ell^i\|_2 \notag\\
&\qquad \leq B_{\ell} \cdot  \|\varphi^i_{\ell-1}\circ\cdots\circ\varphi^i_1(x) \|_2 + A_\ell\notag\\
&\qquad \leq \|x\|_2\cdot  \prod_{j = 1}^\ell B_j + \sum_{j = 1}^\ell A_j \prod_{i =j+1}^\ell B_i,
\#
where the first inequality comes from the triangle inequality, and the second inequality comes from the bounded spectral norm of $W_j^i$, while the last inequality simply applies the previous arguments recursively. Therefore, combining \eqref{eq:lip_w1}, we have
\#\label{eq:lip_w2}
& \|\varphi(x,W^1, v^1) - \varphi(x,W^2, v^2) \|_2  \leq \biggl( B\cdot \prod_{j = 1}^L B_j + \sum_{j = 1}^L A_j \prod_{i =j+1}^L B_i \biggr) \cdot \|W_{L+1}^1 - W_{L+1}^2\|_F\notag\\
&\qquad\qquad\qquad\qquad\qquad + B_{L+1}\cdot   \|\varphi^1_L\circ\cdots\circ\varphi^1_1(x) - \varphi^2_L\circ\cdots\circ\varphi^2_1(x) \|_2. 
\#
Similarly, by triangular inequality, we have
\#\label{eq:lip_w3}
& \|\varphi^1_L\circ\cdots\circ\varphi^1_1(x) - \varphi^2_L\circ\cdots\circ\varphi^2_1(x) \|_2\notag\\
&\qquad \leq  \|\varphi^1_L\circ\varphi^1_{L-1}\circ\cdots\circ\varphi^1_1(x) - \varphi^2_L\circ\varphi^1_{L-1}\circ\cdots\circ\varphi^1_1(x) \|_2\notag \\
&\qquad\qquad\qquad+  \|\varphi^2_L\circ\varphi^1_{L-1}\circ\cdots\circ\varphi^1_1(x) - \varphi^2_L\circ\varphi^2_{L-1}\circ\cdots\circ\varphi^2_1(x) \|_2\notag\\
&\qquad \leq  \|\varphi^1_L\circ\varphi^1_{L-1}\circ\cdots\circ\varphi^1_1(x) - \varphi^2_L\circ\varphi^1_{L-1}\circ\cdots\circ\varphi^1_1(x) \|_2 \\
&\qquad\qquad\qquad+ B_L\cdot  \|\varphi^1_{L-1}\circ\cdots\circ\varphi^1_1(x) - \varphi^2_{L-1}\circ\cdots\circ\varphi^2_1(x) \|_2,\notag
\#
where the second inequality uses the bounded spectral norm of $W_L$ and $1$-Lipschitz property of $\sigma_{v_L}(\cdot)$. For notational convenience, we further denote $y = \varphi^1_{L-1}\circ\cdots\circ\varphi^1_1(x)$, then 
\$
 \|\varphi_L^1(y) - \varphi_L^2(y) \|_2& =  \| \sigma(W_L^1y-v_L^1) - \sigma(W_L^2y-v_L^2)  \}  \|_2\notag\\
 & \leq \|v_L^1 - v_L^2\|_2 + \| W_L^1 - W_L^2 \|_\F \cdot \|y\|_2,
\$
where the inequality comes from the $1$-Lipschitz property of $\sigma(\cdot)$. 
Moreover, combining \eqref{eq:rev-1}, it holds that
\#\label{eq:lip_w4}
\|\varphi_L^1(y) - \varphi_L^2(y) \|_2 \leq \|v_L^1 - v_L^2\|_2 + \| W_L^1 - W_L^2 \|_\F \cdot \biggl( B \cdot  \prod_{j = 1}^{L-1} B_j + \sum_{j = 1}^{L-1} A_j \prod_{i =j+1}^{L-1} B_i \biggr). 
\#
By \eqref{eq:lip_w3} and \eqref{eq:lip_w4}, we have
\$
& \|\varphi^1_L\circ\cdots\circ\varphi^1_1(x) - \varphi^2_L\circ\cdots\circ\varphi^2_1(x) \|_2\notag\\
&\qquad\leq  \|v_L^1 - v_L^2\|_2 + \| W_L^1 - W_L^2 \|_\F \cdot \biggl( B \cdot  \prod_{j = 1}^{L-1} B_j + \sum_{j = 1}^{L-1} A_j \prod_{i =j+1}^{L-1} B_i \biggr) \notag\\
&\qquad\qquad + B_L\cdot  \|\varphi^1_{L-1}\circ\cdots\circ\varphi^1_1(x) - \varphi^2_{L-1}\circ\cdots\circ\varphi^2_1(x) \|_2\notag\\
&\qquad \leq \sum_{j = 1}^L  \prod_{i = j+1}^L B_i \cdot  \|v_j^1 - v_j^2 \|_2 + \frac{B\cdot \prod_{j = 1}^{L+1}B_j}{\min_j B_j} \cdot \sum_{j = 1}^L A_j \cdot\sum_{j = 1}^L  \|W_j^1 - W_j^2 \|_\F\\
&\qquad\leq \frac{B\cdot \prod_{j = 1}^{L+1}B_j}{\min_j B_j}\cdot \sum_{j = 1}^L A_j \cdot \sum_{j = 1}^L ( \|v_j^1 - v_j^2 \|_2 +   \|W_j^1 - W_j^2 \|_\F ). 
\$
Here in the second inequality we recursively apply the previous arguments. Further combining \eqref{eq:lip_w2}, we obtain that 
\$
& \|\varphi(x,W^1, v^1) - \varphi(x,W^2, v^2) \|_2\\
&\qquad\leq \frac{B\cdot \prod_{j = 1}^{L+1}B_j}{\min_j B_j}\cdot  \sum_{j = 1}^L A_j \cdot  \biggl(\sum_{j = 1}^{L+1}  \|W_j^1 - W_j^2 \|_\F + \sum_{j = 1}^L \|v_j^1 - v_j^2 \|_2  \biggr)\\
&\qquad\leq \frac{B\sqrt{2L+1}\cdot \prod_{j = 1}^{L+1}B_j}{\min_j B_j}\cdot \sum_{j = 1}^L A_j \cdot \sqrt{\sum_{j = 1}^{L+1}\|W_j^1 - W_j^2\|_\F^2 + \sum_{j = 1}^L \|v_j^1 - v_j^2\|_2^2},
\$
where we use Cauchy-Schwarz inequality in the last line. This concludes the proof of the lemma.

\section{Auxiliary Results}

\begin{lemma} \label{lemma:cover_number}
The following statements for entropy hold.
\begin{enumerate}
\item Suppose that $\sup_{g\in\cG} \|g\|_\infty\leq M$, then
\$
\cH_{4M, B}(\sqrt{2}\delta, \cG, \QQ)\leq H_{2, B}(\delta, \cG, \QQ)
\$
for any $\delta > 0$.

\item For $1\leq q<\infty$, and $\QQ$ a distribution, we have
\$
H_{p,B}(\delta, \cG, \QQ)\leq H_\infty(\delta/2, \cG),
\$
for any $\delta > 0$. Here $H_\infty$ is the entropy induced by infinity norm. 

\item Based on the above two statements, suppose that $\sup_{g\in\cG} \|g\|_\infty\leq M$, we have
\$
\cH_{4M, B}(\sqrt{2}\cdot\delta, \cG, \QQ)\leq H_\infty(\delta/2, \cG),
\$
by taking $p = 2$. 
\end{enumerate} 
\end{lemma}
\begin{proof}
See  \cite{van2000empirical} for a detailed proof. 
\end{proof}

\begin{lemma}\label{lemma:cover_number_nn}
The entropy of the neural network set defined in \eqref{def:nn_set} satisfies
\$
H_\infty [\delta, \Phi_M(L, {p}, s) ]\leq (s+1)\log (2\delta^{-1}(L+1)V^2 ),
\$
where $V=\prod_{l = 0}^{L+1}(p_l+1)$. 
\end{lemma}
\begin{proof}
See \cite{schmidt2017nonparametric} for a detailed proof. 
\end{proof}

\begin{theorem}\label{thm:vandegeer}
Assume that $\sup_{g\in \cG}\rho_K(g)\leq R$. Take $a$, $C$, $C_0$, and $C_1$ satisfying that $a\leq C_1\sqrt{n}R^2/K$, $a\leq 8\sqrt{n}R$, $a\geq C_0\cdot [ \int_0^R H_{K, B}^{1/2}(u, \cG, \PP)du\vee R  ]$, and $C_0^2\geq C^2(C_1 + 1)$. 
It holds that
\$
\PP [\sup_{g\in\cG} |\EE_{\PP_n}(g)   - \EE_{\PP}(g) |\geq a\cdot n^{-1/2} ]\leq C\exp \biggl( -\frac{a^2}{C^2(C_1+1)R^2} \biggr).
\$
\end{theorem}
\begin{proof}
See  \cite{van2000empirical} for a detailed proof.
\end{proof}

\begin{lemma}\label{lemma:bound_rho}
Suppose that $\|g\|_\infty\leq K$, and $\|g\|\leq R$, then $\rho_{2K,\PP}^2(g)\leq 2R^2$. Moreover, for any $K'\geq K$, we have $\rho_{2K',\PP}^2(g)\leq 2R^2$. 
\end{lemma}
\begin{proof}
See \cite{van2000empirical} for a detailed proof.
\end{proof}

\begin{theorem}\label{theorem:approx}
For any function $f$ in the H\"older ball $\cC_d^\beta([0,1]^d, K)$ and any integers $m\geq 1$ and $N\geq (\beta+ 1)^d\vee (K+1)$, there exists a network $\tilde f\in \Phi(L, (d, 12dN, \ldots, 12dN, 1), s)$ with number of layers $L = 8 + (m+5)(1+\lceil \log_2 d\rceil)$ and number of parameters $s\leq 94 d^2(\beta+1)^{2d} N(m+6)(1+\lceil \log_2 d\rceil)$, such that 
\$
\| \tilde f  - f\|_{L^\infty ([0,1]^d)}\leq (2K+1)3^{d+1}N2^{-m} + K2^\beta N^{-\beta/d}. 
\$
\end{theorem}
\begin{proof}
See \cite{schmidt2017nonparametric} for a detailed proof. 
\end{proof}

\begin{lemma}\label{lemma:strong_convex}
If the function $f$ is strongly convex with parameter $\mu_0 > 0$ and has Lipschitz continuous gradient with parameter $L_0>0$, then the Fenchel duality $f^\dagger$ of $f$ is $1/L_0$-strongly convex and has $1/\mu_0$-Lipschitz continuous gradient (therefore, $f^\dagger$ itself is Lipschitz continuous). 
\end{lemma}
\begin{proof}
See \cite{zhou2018fenchel} for a detailed proof. 
\end{proof}

\section{Experiment details}

 To evaluate the performance of our mechanism on the MNIST and CIFAR-10 test dataset, we first observe that for high-dimensional data, the optimization task in step 1 may fail to converge to the global (or a high-quality local) optimum. Adopting a fixed form of $\hat{t}$ can still guarantee incentive properties of our mechanism and also consumes less time. Thus, we skip Step 1 in Algorithms 1, 2, and instead adopt $\hat{t}$ from the existing literature.

\subsection{Evaluation with ground-truth verification}

To estimate distributions w.r.t. images, we borrow a practical trick as implemented in \cite{nowozin2016f}: let's denote a public discriminator as $D$ which has been pre-trained on corresponding training dataset.  Given a batch of clean (ground-truth) images $\{x_i \}_{i=1}^n$, agent $\mathbf{A}$'s corresponding untruthful reports $\{\tilde{x}_i \}_{i=1}^n$, the score of $\mathbf{A}$'s reports is calculated by: 
\[
S(\{\tilde{x}_i \}_{i=1}^n, \{x_i \}_{i=1}^n)=a-\dfrac{b}{n}\cdot \sum_{i=1}^n\Big[\hat{\tran}(D(x_i))-f^{\dag}(\hat{\tran}(D(\tilde{x}_i)))\Big]
\]

\subsection{Evaluation without ground-truth verification}
Suppose we have access to a batch of peer reported images $\{\bar{x}_i \}_{i=1}^n$, agent $\mathbf{A}$'s corresponding untruthful reports $\{\tilde{x}_i \}_{i=1}^n$. For $\PP_n=\{\tilde{x}_i\}_{i=1}^n$, $\QQ_n=\{\bar{x}_i\}_{i=1}^n$, we use $(D(\tilde{x}_i)+ D(\bar{x}_i))/2$ to estimate the distribution $x\thicksim \PP \oplus \QQ$, and  $D(\tilde{x}_i) \cdot D(\bar{x}_i)$ is the estimation of $x\thicksim \PP \times \QQ$. The score of $\mathbf{A}$'s reports is calculated by: 
\[
S(\{\tilde{x}_i \}_{i=1}^n, \{\bar{x}_i \}_{i=1}^n)=a+\dfrac{b}{n}\cdot \sum_{i=1}^n\Big[\hat{\tran}\Big(\dfrac{D(\tilde{x}_i)+ D(\bar{x}_i)}{2} \Big)-f^{\dag}(\hat{\tran}(D(\tilde{x}_i) \cdot D(\bar{x}_i)))\Big]
\]

\subsection{Computing infrastructure}
In our experiments, we use a GPU cluster (8 TITAN V GPUs and 16 GeForce GTX 1080 GPUs) for training and evaluation.

%% file: main.bbl
\begin{thebibliography}{56}
\providecommand{\natexlab}[1]{#1}
\providecommand{\url}[1]{\texttt{#1}}
\expandafter\ifx\csname urlstyle\endcsname\relax
  \providecommand{\doi}[1]{doi: #1}\else
  \providecommand{\doi}{doi: \begingroup \urlstyle{rm}\Url}\fi

\bibitem[Abernethy and Frongillo(2012)]{abernethy2012characterization}
Jacob~D Abernethy and Rafael~M Frongillo.
\newblock A characterization of scoring rules for linear properties.
\newblock In \emph{Conference on Learning Theory}, pages 27--1, 2012.

\bibitem[Arjovsky et~al.(2017)Arjovsky, Chintala, and
  Bottou]{arjovsky2017wasserstein}
Martin Arjovsky, Soumith Chintala, and L{\'e}on Bottou.
\newblock Wasserstein {GAN}.
\newblock \emph{arXiv preprint arXiv:1701.07875}, 2017.

\bibitem[Arora et~al.(2017)Arora, Ge, Liang, Ma, and
  Zhang]{arora2017generalization}
Sanjeev Arora, Rong Ge, Yingyu Liang, Tengyu Ma, and Yi~Zhang.
\newblock Generalization and equilibrium in generative adversarial nets
  ({GAN}s).
\newblock In \emph{International Conference on Machine Learning}, pages
  224--232, 2017.

\bibitem[Bellemare et~al.(2017)Bellemare, Danihelka, Dabney, Mohamed,
  Lakshminarayanan, Hoyer, and Munos]{bellemare2017cramer}
Marc~G Bellemare, Ivo Danihelka, Will Dabney, Shakir Mohamed, Balaji
  Lakshminarayanan, Stephan Hoyer, and R{\'e}mi Munos.
\newblock The {C}ramer distance as a solution to biased {W}asserstein
  gradients.
\newblock \emph{arXiv preprint arXiv:1705.10743}, 2017.

\bibitem[Brier(1950)]{Brier:50}
Glenn~W. Brier.
\newblock Verification of forecasts expressed in terms of probability.
\newblock \emph{Monthly Weather Review}, 78\penalty0 (1):\penalty0 1--3, 1950.

\bibitem[Broniatowski and Keziou(2004)]{broniatowski2004parametric}
Michel Broniatowski and Amor Keziou.
\newblock Parametric estimation and tests through divergences.
\newblock Technical report, Citeseer, 2004.

\bibitem[Broniatowski and Keziou(2009)]{broniatowski2009parametric}
Michel Broniatowski and Amor Keziou.
\newblock Parametric estimation and tests through divergences and the duality
  technique.
\newblock \emph{Journal of Multivariate Analysis}, 100\penalty0 (1):\penalty0
  16--36, 2009.

\bibitem[Bu et~al.(2018)Bu, Zou, Liang, and Veeravalli]{bu2018estimation}
Yuheng Bu, Shaofeng Zou, Yingbin Liang, and Venugopal~V Veeravalli.
\newblock Estimation of {KL} divergence: Optimal minimax rate.
\newblock \emph{IEEE Transactions on Information Theory}, 64\penalty0
  (4):\penalty0 2648--2674, 2018.

\bibitem[De~Alfaro et~al.(2016)De~Alfaro, Shavlovsky, and
  Polychronopoulos]{de2016incentives}
Luca De~Alfaro, Michael Shavlovsky, and Vassilis Polychronopoulos.
\newblock Incentives for truthful peer grading.
\newblock \emph{arXiv preprint arXiv:1604.03178}, 2016.

\bibitem[Deng et~al.(2009)Deng, Dong, Socher, Li, Li, and
  Fei-Fei]{deng2009imagenet}
Jia Deng, Wei Dong, Richard Socher, Li-Jia Li, Kai Li, and Li~Fei-Fei.
\newblock Image{N}et: A large-scale hierarchical image database.
\newblock In \emph{Conference on Computer Vision and Pattern Recognition},
  pages 248--255, 2009.

\bibitem[Ding et~al.(2019)Ding, Wang, and Jin]{ding2019advertorch}
Gavin~Weiguang Ding, Luyu Wang, and Xiaomeng Jin.
\newblock Advertorch v0.1: An adversarial robustness toolbox based on pytorch.
\newblock \emph{arXiv preprint arXiv:1902.07623}, 2019.

\bibitem[Donsker and Varadhan(1975)]{donsker1975asymptotic}
Monroe~D Donsker and SR~Srinivasa Varadhan.
\newblock Asymptotic evaluation of certain {M}arkov process expectations for
  large time. {I}.
\newblock \emph{Communications on Pure and Applied Mathematics}, 28\penalty0
  (1):\penalty0 1--47, 1975.

\bibitem[Frongillo and Kash(2015{\natexlab{a}})]{frongillo2015elicitation}
Rafael Frongillo and Ian Kash.
\newblock On elicitation complexity.
\newblock In \emph{Advances in Neural Information Processing Systems}, pages
  3258--3266, 2015{\natexlab{a}}.

\bibitem[Frongillo and Kash(2015{\natexlab{b}})]{frongillo2015vector}
Rafael Frongillo and Ian~A Kash.
\newblock Vector-valued property elicitation.
\newblock In \emph{Conference on Learning Theory}, pages 710--727,
  2015{\natexlab{b}}.

\bibitem[Gao et~al.(2016)Gao, Wright, and Leyton-Brown]{gao2016incentivizing}
Alice Gao, James~R Wright, and Kevin Leyton-Brown.
\newblock Incentivizing evaluation via limited access to ground truth:
  Peer-prediction makes things worse.
\newblock \emph{arXiv preprint arXiv:1606.07042}, 2016.

\bibitem[Gao et~al.(2019)Gao, Yao, and Zhu]{gao2019generative}
Chao Gao, Yuan Yao, and Weizhi Zhu.
\newblock Generative adversarial nets for robust scatter estimation: A proper
  scoring rule perspective.
\newblock \emph{arXiv preprint arXiv:1903.01944}, 2019.

\bibitem[Gao et~al.(2017)Gao, Oh, and Viswanath]{gao2017density}
Weihao Gao, Sewoong Oh, and Pramod Viswanath.
\newblock Density functional estimators with k-nearest neighbor bandwidths.
\newblock In \emph{International Symposium on Information Theory}, pages
  1351--1355, 2017.

\bibitem[Gneiting and Raftery(2007)]{Gneiting:07}
Tilmann Gneiting and Adrian~E. Raftery.
\newblock Strictly proper scoring rules, prediction, and estimation.
\newblock \emph{Journal of the American Statistical Association}, 102\penalty0
  (477):\penalty0 359--378, 2007.

\bibitem[Goodfellow et~al.(2014)Goodfellow, Pouget-Abadie, Mirza, Xu,
  Warde-Farley, Ozair, Courville, and Bengio]{goodfellow2014generative}
Ian Goodfellow, Jean Pouget-Abadie, Mehdi Mirza, Bing Xu, David Warde-Farley,
  Sherjil Ozair, Aaron Courville, and Yoshua Bengio.
\newblock Generative adversarial nets.
\newblock In \emph{Advances in neural information processing systems}, pages
  2672--2680, 2014.

\bibitem[Gulrajani et~al.(2017)Gulrajani, Ahmed, Arjovsky, Dumoulin, and
  Courville]{gulrajani2017improved}
Ishaan Gulrajani, Faruk Ahmed, Martin Arjovsky, Vincent Dumoulin, and Aaron~C
  Courville.
\newblock Improved training of {W}asserstein {GAN}s.
\newblock In \emph{Advances in Neural Information Processing Systems}, pages
  5767--5777, 2017.

\bibitem[Han et~al.(2016)Han, Jiao, and Weissman]{han2016minimax}
Yanjun Han, Jiantao Jiao, and Tsachy Weissman.
\newblock Minimax rate-optimal estimation of divergences between discrete
  distributions.
\newblock \emph{arXiv preprint arXiv:1605.09124}, 2016.

\bibitem[Heusel et~al.(2017)Heusel, Ramsauer, Unterthiner, Nessler, and
  Hochreiter]{Fid}
Martin Heusel, Hubert Ramsauer, Thomas Unterthiner, Bernhard Nessler, and Sepp
  Hochreiter.
\newblock Gans trained by a two time-scale update rule converge to a local nash
  equilibrium.
\newblock \emph{arXiv preprint arXiv:1706.08500}, 2017.

\bibitem[Jose et~al.(2006)Jose, Nau, and Winkler]{Jose:06}
Victor~Richmond Jose, Robert~F. Nau, and Robert~L. Winkler.
\newblock Scoring rules, generalized entropy and utility maximization.
\newblock Working Paper, Fuqua School of Business, Duke University, 2006.

\bibitem[Kanamori et~al.(2011)Kanamori, Suzuki, and Sugiyama]{kanamori2011f}
Takafumi Kanamori, Taiji Suzuki, and Masashi Sugiyama.
\newblock $ f $-divergence estimation and two-sample homogeneity test under
  semiparametric density-ratio models.
\newblock \emph{IEEE Transactions on Information Theory}, 58\penalty0
  (2):\penalty0 708--720, 2011.

\bibitem[Kong and Schoenebeck(2018)]{kong2018water}
Yuqing Kong and Grant Schoenebeck.
\newblock Water from two rocks: Maximizing the mutual information.
\newblock In \emph{Conference on Economics and Computation}, pages 177--194,
  2018.

\bibitem[Kong and Schoenebeck(2019)]{kong2019information}
Yuqing Kong and Grant Schoenebeck.
\newblock An information theoretic framework for designing information
  elicitation mechanisms that reward truth-telling.
\newblock \emph{Transactions on Economics and Computation}, 7\penalty0
  (1):\penalty0 2, 2019.

\bibitem[Kong et~al.(2016)Kong, Ligett, and Schoenebeck]{kong2016putting}
Yuqing Kong, Katrina Ligett, and Grant Schoenebeck.
\newblock Putting peer prediction under the micro (economic) scope and making
  truth-telling focal.
\newblock In \emph{International Conference on Web and Internet Economics},
  pages 251--264. Springer, 2016.

\bibitem[Krizhevsky(2009)]{cifar}
A.~Krizhevsky.
\newblock Learning multiple layers of features from tiny images.
\newblock Master’s thesis, Department of Computer Science, University of
  Toronto, 2009.

\bibitem[Krizhevsky et~al.(2012)Krizhevsky, Sutskever, and
  Hinton]{krizhevsky2012imagenet}
Alex Krizhevsky, Ilya Sutskever, and Geoffrey~E Hinton.
\newblock Imagenet classification with deep convolutional neural networks.
\newblock In \emph{Advances in neural information processing systems}, pages
  1097--1105, 2012.

\bibitem[Lambert et~al.(2008)Lambert, Pennock, and
  Shoham]{lambert2008eliciting}
N.S. Lambert, D.M. Pennock, and Y.~Shoham.
\newblock Eliciting properties of probability distributions.
\newblock In \emph{Conference on Electronic Commerce}, pages 129--138, 2008.

\bibitem[LeCun et~al.(1998)LeCun, Bottou, Bengio, and Haffner]{mnist}
Yann LeCun, L{\'e}on Bottou, Yoshua Bengio, and Patrick Haffner.
\newblock Gradient-based learning applied to document recognition.
\newblock \emph{Proceedings of the IEEE}, 86\penalty0 (11):\penalty0
  2278--2324, 1998.

\bibitem[Lee and Park(2006)]{lee2006estimation}
Young~Kyung Lee and Byeong~U Park.
\newblock Estimation of {K}ullback--{L}eibler divergence by local likelihood.
\newblock \emph{Annals of the Institute of Statistical Mathematics},
  58\penalty0 (2):\penalty0 327--340, 2006.

\bibitem[Li et~al.(2018)Li, Lu, Wang, Haupt, and Zhao]{li2018tighter}
Xingguo Li, Junwei Lu, Zhaoran Wang, Jarvis Haupt, and Tuo Zhao.
\newblock On tighter generalization bound for deep neural networks: {CNN}s,
  {R}es{N}ets, and beyond.
\newblock \emph{arXiv preprint arXiv:1806.05159}, 2018.

\bibitem[Liang(2018)]{liang2018well}
Tengyuan Liang.
\newblock On how well generative adversarial networks learn densities:
  Nonparametric and parametric results.
\newblock \emph{arXiv preprint arXiv:1811.03179}, 2018.

\bibitem[Liu et~al.(2017)Liu, Bousquet, and Chaudhuri]{liu2017approximation}
Shuang Liu, Olivier Bousquet, and Kamalika Chaudhuri.
\newblock Approximation and convergence properties of generative adversarial
  learning.
\newblock In \emph{Advances in Neural Information Processing Systems}, pages
  5545--5553, 2017.

\bibitem[Matheson and Winkler(1976)]{Matheson:76}
James~E. Matheson and Robert~L. Winkler.
\newblock Scoring rules for continuous probability distributions.
\newblock \emph{Management Science}, 22\penalty0 (10):\penalty0 1087--1096,
  1976.

\bibitem[McCarthy(1956)]{McCarthy:56}
John McCarthy.
\newblock Measures of the value of information.
\newblock \emph{Proceedings of the National Academy of Sciences of the United
  States of America}, 42\penalty0 (9):\penalty0 654--655, 1956.

\bibitem[Mohri et~al.(2018)Mohri, Rostamizadeh, and
  Talwalkar]{mohri2018foundations}
Mehryar Mohri, Afshin Rostamizadeh, and Ameet Talwalkar.
\newblock \emph{Foundations of {M}achine {L}earning}.
\newblock MIT press, 2018.

\bibitem[Nguyen et~al.(2010)Nguyen, Wainwright, and
  Jordan]{nguyen2010estimating}
XuanLong Nguyen, Martin~J Wainwright, and Michael~I Jordan.
\newblock Estimating divergence functionals and the likelihood ratio by convex
  risk minimization.
\newblock \emph{IEEE Transactions on Information Theory}, 56\penalty0
  (11):\penalty0 5847--5861, 2010.

\bibitem[Nowozin et~al.(2016)Nowozin, Cseke, and Tomioka]{nowozin2016f}
Sebastian Nowozin, Botond Cseke, and Ryota Tomioka.
\newblock f-gan: Training generative neural samplers using variational
  divergence minimization.
\newblock In \emph{Advances in neural information processing systems}, pages
  271--279, 2016.

\bibitem[Ruderman et~al.(2012)Ruderman, Reid, Garc{\'\i}a-Garc{\'\i}a, and
  Petterson]{ruderman2012tighter}
Avraham Ruderman, Mark Reid, Dar{\'\i}o Garc{\'\i}a-Garc{\'\i}a, and James
  Petterson.
\newblock Tighter variational representations of $f$-divergences via
  restriction to probability measures.
\newblock \emph{arXiv preprint arXiv:1206.4664}, 2012.

\bibitem[Savage(1971)]{Savage:71}
Leonard~J. Savage.
\newblock Elicitation of personal probabilities and expectations.
\newblock \emph{Journal of the American Statistical Association}, 66\penalty0
  (336):\penalty0 783--801, 1971.

\bibitem[Schmidt-Hieber(2017)]{schmidt2017nonparametric}
Johannes Schmidt-Hieber.
\newblock Nonparametric regression using deep neural networks with relu
  activation function.
\newblock \emph{arXiv preprint arXiv:1708.06633}, 2017.

\bibitem[Schoenebeck and Yu(2020)]{schoenebeck2020learning}
Grant Schoenebeck and Fang-Yi Yu.
\newblock Learning and strongly truthful multi-task peer prediction: A
  variational approach.
\newblock \emph{arXiv preprint arXiv:2009.14730}, 2020.

\bibitem[Srivastava et~al.(2014)Srivastava, Hinton, Krizhevsky, Sutskever, and
  Salakhutdinov]{srivastava2014dropout}
Nitish Srivastava, Geoffrey Hinton, Alex Krizhevsky, Ilya Sutskever, and Ruslan
  Salakhutdinov.
\newblock Dropout: a simple way to prevent neural networks from overfitting.
\newblock \emph{The Journal of Machine Learning Research}, 15\penalty0
  (1):\penalty0 1929--1958, 2014.

\bibitem[Steinwart et~al.(2014)Steinwart, Pasin, Williamson, and
  Zhang]{steinwart2014elicitation}
Ingo Steinwart, Chlo{\'e} Pasin, Robert Williamson, and Siyu Zhang.
\newblock Elicitation and identification of properties.
\newblock In \emph{Conference on Learning Theory}, pages 482--526, 2014.

\bibitem[Stone(1982)]{stone1982optimal}
Charles~J Stone.
\newblock Optimal global rates of convergence for nonparametric regression.
\newblock \emph{The annals of statistics}, pages 1040--1053, 1982.

\bibitem[Sugiyama et~al.(2012)Sugiyama, Suzuki, and
  Kanamori]{sugiyama2012density}
Masashi Sugiyama, Taiji Suzuki, and Takafumi Kanamori.
\newblock \emph{Density {R}atio {E}stimation in {M}achine {L}earning}.
\newblock Cambridge University Press, 2012.

\bibitem[Suzuki et~al.(2008)Suzuki, Sugiyama, Sese, and
  Kanamori]{suzuki2008approximating}
Taiji Suzuki, Masashi Sugiyama, Jun Sese, and Takafumi Kanamori.
\newblock Approximating mutual information by maximum likelihood density ratio
  estimation.
\newblock In \emph{New challenges for feature selection in data mining and
  knowledge discovery}, pages 5--20, 2008.

\bibitem[van~de Geer and van~de Geer(2000)]{van2000empirical}
Sara~A van~de Geer and Sara van~de Geer.
\newblock \emph{Empirical Processes in M-estimation}, volume~6.
\newblock Cambridge university press, 2000.

\bibitem[Van~der Walt et~al.(2014)Van~der Walt, Sch{\"o}nberger,
  Nunez-Iglesias, Boulogne, Warner, Yager, Gouillart, and Yu]{scikit-image}
Stefan Van~der Walt, Johannes~L Sch{\"o}nberger, Juan Nunez-Iglesias,
  Fran{\c{c}}ois Boulogne, Joshua~D Warner, Neil Yager, Emmanuelle Gouillart,
  and Tony Yu.
\newblock scikit-image: image processing in python.
\newblock \emph{PeerJ}, 2:\penalty0 e453, 2014.

\bibitem[Wang et~al.(2005)Wang, Kulkarni, and Verd{\'u}]{wang2005divergence}
Qing Wang, Sanjeev~R Kulkarni, and Sergio Verd{\'u}.
\newblock Divergence estimation of continuous distributions based on
  data-dependent partitions.
\newblock \emph{IEEE Transactions on Information Theory}, 51\penalty0
  (9):\penalty0 3064--3074, 2005.

\bibitem[Wang et~al.(2009)Wang, Kulkarni, and Verd{\'u}]{wang2009divergence}
Qing Wang, Sanjeev~R Kulkarni, and Sergio Verd{\'u}.
\newblock Divergence estimation for multidimensional densities via $ k
  $-nearest-neighbor distances.
\newblock \emph{IEEE Transactions on Information Theory}, 55\penalty0
  (5):\penalty0 2392--2405, 2009.

\bibitem[Winkler(1969)]{Win:69}
Robert~L. Winkler.
\newblock Scoring rules and the evaluation of probability assessors.
\newblock \emph{Journal of the American Statistical Association}, 64\penalty0
  (327):\penalty0 1073--1078, 1969.

\bibitem[Zhang and Grabchak(2014)]{zhang2014nonparametric}
Zhiyi Zhang and Michael Grabchak.
\newblock Nonparametric estimation of {K}ullback--{L}eibler divergence.
\newblock \emph{Neural computation}, 26\penalty0 (11):\penalty0 2570--2593,
  2014.

\bibitem[Zhou(2018)]{zhou2018fenchel}
Xingyu Zhou.
\newblock On the {F}enchel duality between strong convexity and {L}ipschitz
  continuous gradient.
\newblock \emph{arXiv preprint arXiv:1803.06573}, 2018.

\end{thebibliography}
